\documentclass{article}
\usepackage{kr}

\usepackage{times}
\usepackage{soul}
\usepackage{url}
\usepackage[hidelinks]{hyperref}
\usepackage[utf8]{inputenc}
\usepackage[small]{caption}
\usepackage{graphicx}
\usepackage{amsmath}
\usepackage{amsthm}
\usepackage{booktabs}
\usepackage{algorithm}
\usepackage{algorithmic}
\newboolean{includeMemo}
\setboolean{includeMemo}{true} 

\newcommand{\memo}[1]{\ifthenelse{\boolean{includeMemo}}{\todo[inline,caption={},color=green!20!]{#1}}}
\newcommand{\memob}[1]{\ifthenelse{\boolean{includeMemo}}{\todo[inline,caption={},color=blue!20!]{#1}}}

\newcommand{\ignore}[1]{}

\newcommand{\squishlist}{
\begin{list}{{{\small{$\bullet$}}}}
{\setlength{\itemsep}{3pt}      
\setlength{\parsep}{3pt}
\setlength{\topsep}{3pt}       
\setlength{\partopsep}{3pt}
\setlength{\leftmargin}{1em} 
\setlength{\labelwidth}{1em}
\setlength{\labelsep}{0.5em} } }
\newcommand{\squishend}{  \end{list}}

\newcommand{\squishenum}{
\begin{list}{$\bullet$}{ 
    \setlength{\itemsep}{1pt}
    \setlength{\parsep}{0pt}
    \setlength{\topsep}{1.5pt}
    \setlength{\partopsep}{0pt}
    \setlength{\leftmargin}{2em}
    \setlength{\labelwidth}{1.5em}
    \setlength{\labelsep}{0.5em} } }

\newtheorem{property}{Property}
\newtheorem{theorem}{Theorem}
\newtheorem{proposition}{Proposition}
\newtheorem{corollary}{Corollary}

\newtheorem{definition}{Definition}
\newtheorem{example}{Example}
\newtheorem{dummytheorem}{Theorem}
\newtheorem{dummyproposition}{Proposition}
\newtheorem{dummycorollary}{Corollary}

\newcommand{\Args}{\ensuremath{\mathcal{X}}}
\newcommand{\Atts}{\ensuremath{\mathcal{A}}}
\newcommand{\Supps}{\ensuremath{\mathcal{S}}}
\newcommand{\Rels}{\ensuremath{\mathcal{R}}}
\newcommand{\BS}{\ensuremath{\tau}}
\newcommand{\SF}{\ensuremath{\sigma}}

\newcommand{\graph}{\ensuremath{\mathcal{G}}}

\newcommand{\CST}{\ensuremath{\mathcal{T}}}

\usepackage{savesym}
\savesymbol{Bbbk}
\savesymbol{openbox}

\usepackage{xcolor}
\usepackage{amssymb}
\usepackage{MnSymbol}
\usepackage{makecell}
\usepackage{float}
\usepackage{array}
\usepackage{placeins}
\usepackage{multirow}
\usepackage{tabularx}
\usepackage{todonotes}

\restoresymbol{NEW}{Bbbk}
\restoresymbol{NEW}{openbox}
\usepackage{tcolorbox}

\title{A Methodology for Incompleteness-Tolerant and Modular\\  Gradual Semantics for Argumentative Statement Graphs}

\author{
Antonio Rago$^{1,2}$\footnote{Equal Contribution.}\!\!\!
\and
\!Stylianos Loukas Vasileiou$^{3,4*}$\!\!\!\and
\!Francesca Toni$^2$\!\!\and
\!Tran Cao Son$^3$\!\!\And\!\!\!
William Yeoh$^4$\!\\
\affiliations
$^1$King's College London\\
$^2$Imperial College London\\
$^3$New Mexico State University\\
$^4$Washington University in St. Louis\\
\emails
antonio.rago@kcl.ac.uk,
\{stelios,stran\}@nmsu.edu,
ft@imperial.ac.uk,
wyeoh@wustl.edu
}

\begin{document}

\maketitle

\begin{abstract}
    Gradual semantics (GS) have demonstrated great potential in argumentation, in particular for deploying \emph{quantitative bipolar argumentation frameworks} (QBAFs) in a number of real-world settings, from judgmental forecasting to explainable AI%\FT{, where arguments can be treated as abstract entities}
    . In this paper, we provide a novel methodology for obtaining GS for
    \emph{statement graphs}, a form of structured argumentation framework, where
    arguments and relations between them are built from logical statements. 
    %using logical rules which allow relations to and from other arguments to be obtained.
    %\delete{arguments have a structure and relations between are obtained based on this structure.} 
    %are obtained from finer-grained building blocks.
    %    the building blocks of arguments and relations between them are known, unlike in QBAFs, where arguments are abstract entities. 
    % \delete{Differently from existing approaches for obtaining gradual semantics for SGs,}
    Our methodology differs from existing approaches in the literature in two main ways. First, it naturally accommodates incomplete information, so that arguments with partially specified premises can play a meaningful role in the evaluation.
    Second, it is modularly defined to  leverage on any GS for QBAFs.
     We also define a set of novel properties for our GS and study their suitability 
    alongside a set of existing properties (adapted to our setting)
    %. Finally, we provide a comprehensive theoretical analysis assessing the instantiations \FT{on existing and novel properties}, demonstrating their advantages over existing GS for %QBAFs and structured argumentation
    for two instantiations of our GS, demonstrating their advantages over existing approaches.
    
\end{abstract}

\section{Introduction}
\label{sec:intro}

Argumentation frameworks have emerged as powerful tools for reasoning about and resolving conflicting information in complex environments (see~\cite{handbook,Atkinson_17} for overviews). In recent years, \emph{gradual semantics} (GS) have shown great promise in extending these frameworks, particularly in the realm of \emph{quantitative bipolar argumentation frameworks} (QBAFs)~\cite{Baroni_18}, where arguments have weights and %they include two relations of 
may be related by \emph{attack} and \emph{support} relations. %Such bipolar 
These approaches thus allow for both negative and positive, respectively, influences between arguments to be accounted for, arguably aligning more closely with human judgment than traditional argumentation frameworks~\cite{Polberg_18}. GS have found applications in diverse (and often human-centric) areas, from \emph{judgmental forecasting}~\cite{Irwin_22} to \emph{explainable AI}~\cite{Vassiliades_21,Cyras_21%,Guo_23
}, by allowing for quantitative, and thus in some cases more nuanced, evaluations of arguments beyond the traditional accept/reject dichotomy in extension-based semantics~\cite{Dung_95,Cayrol_05}. These quantitative evaluations, typically referred to as \emph{strengths}, can be particularly useful in cases where the information in argumentation frameworks is incomplete, since 
strengths can account for %this 
the resulting uncertainty %can be accounted for 
quantitatively.

However, most of the existing work on GS (see \cite{Baroni_19} for an overview) has focused on formalisms based on 
%the \emph{abstract argumentation} framework 
abstract arguments \cite{Dung_95}, %where arguments are
treated as atomic entities without internal structure. This abstraction, while useful in many contexts, fails to capture the rich logical structure that is often present in real-world arguments. Meanwhile, various forms of \emph{structured argumentation frameworks} (see~\cite{besnard2014introduction} for an overview) 
%e.g., \emph{assumption-based argumentation} (ABA) \cite{Cyras_17} or \emph{ASPIC$^+$} \cite{Modgil_14}}, which explicitly 
allow to represent the internal composition of arguments and the relationships between their components. They
%and 
thus
%may 
offer a more expressive and realistic model in agent-based domains, such as user modeling ~\cite{hadoux2023strategic}, scientific debates~\cite{cramer2020structured}, and model reconciliation settings~\cite{vasileiou2024dialectical}. However, this greater reliance on the logical (e.g., background or contextual) knowledge can lead to more problematic cases where this information is incomplete, where the \emph{completeness} of a statement (as we formally define in §\ref{sec:main}) refers to its grounding in facts. 
Yet, 
%\delete{the development of gradual semantics for structured argumentation has received little attention to date}
it is only recently that the study of GS for structured argumentation has received %more 
 attention, e.g., for \emph{deductive argumentation} \cite{besnard2014constructing} by Heyninck et al.~\shortcite{Heyninck_23}, for \emph{assumption-based argumentation} (ABA)~\cite{Cyras_17} by Skiba et al.~\shortcite{Skiba_23},
 for \emph{ASPIC$^+$}~\cite{Modgil_14} by Spaans~\shortcite{Spaans_21} and Prakken~\shortcite{Prakken_24},
and for a restricted type of structured argumentation in the form of \emph{statement graphs} (SGs) \cite{hecham2018flexible} by Jedwabny \emph{et al.}~\shortcite{Jedwabny_20}. %\FT{In this paper, we also focus on SGs and provide a novel approach to GS therefor.}

We selected SGs as our targeted %the 
form of structured argumentation in this paper% we focus on here
. We chose %choose %SGs as a 
this starting point for our analysis, rather than other forms of structured argumentation, because they naturally accommodate attack and support relations. Their GS are thus naturally relatable to those of QBAFs, which are the most widely studied form of argumentation when it comes to GS \cite{Baroni_19}. This allows us %for 
to undertake an interesting analysis of properties of our GS in direct comparison with those for QBAFs, given their widespread application.

To illustrate the need for GS in structured argumentation,
Figure~\ref{fig:SG} %\delete{illustrates}
shows a debate about climate change represented as an SG. %which uses
SGs use
%A notable exception is the GS proposed by \AR{Jedwabny \emph{et al.}~}\shortcite{Jedwabny_20}, which uses %structured arguments 
a restricted type of structured arguments in the form of \emph{statements}, % within \emph{statement graphs} (SGs) \cite{hecham2018flexible},  
each consisting of a premise (a compound of literals, which may possibly be $\top$, i.e., true) and a claim (a literal). 
%The SG in the figure represents a debate about climate change, where different pieces of evidence support or challenge various claims. 
A typical statement ($\alpha_1$) may combine evidence about rising CO$_2$ emissions ($a$) and increasing global temperatures ($b$), given that they are both supported by facts (%see
$\alpha_2$ and $\alpha_3$), to support the claim that climate change is human-caused ($c$). Such an SG may be the result of %applying 
mining arguments %mining to 
from text, e.g., as overviewed in ~\cite{lawrence-reed-2019-argument}. 
Such argument mining techniques coupled with suitable GS for SGs could allow for %better 
more expressive representation of arguments while handling incomplete information in the modelling of debates, a rapidly growing research area utilising many forms of argumentation \cite{Visser_20,Goffredo_23,Ruiz-Dolz_23,Budan_23,Bistarelli_23,Bezou-Vrakatseli_24}, especially given the groundbreaking potential of large language models \cite{Oluokun_24}.

\begin{figure}[t]
    \centering
    \includegraphics[width=0.5\textwidth]{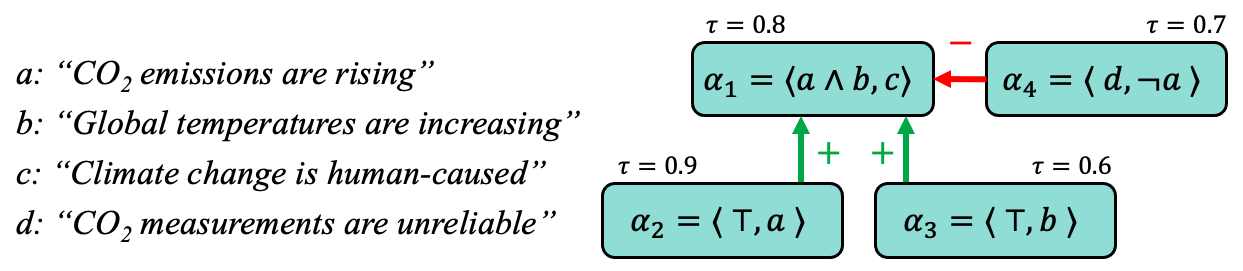}
    \caption{Example SG %for the running example 
    (with arbitrary weights $\BS$, see \S\ref{sec:prelim}). Green edges with a plus %sign 
    denote support; red edges with a minus %sign 
    denote attack. 
    The SG includes incomplete information about the statement $\alpha_4$, since its premise $d$ is not grounded in facts. However, the given weight for $\alpha_4$ provides %partial
    information that our GS can use.}
    \label{fig:SG}
\end{figure}

The GS proposed by Jedwabny \emph{et al.}~\shortcite{Jedwabny_20}
%This GS 
for SGs uses T-norms and T-conorms from %the area of 
fuzzy logic to aggregate the strengths of a statement's components, %offering a way to evaluate structured arguments quantitatively. However, 
but this GS %, like other existing approaches, 
%again 
faces challenges when dealing with incomplete information as it relies on %complete support trees, which assume full 
%an assumption of 
complete knowledge and the presence of a statement's %structure and 
supporting elements. This requirement can limit its applicability in real-world scenarios, where such information is almost always incomplete. %Consider, for instance, 
For example, %\FT{the debate in} 
in Figure~\ref{fig:SG}, 
the strength of the statement $\alpha_1$ %could be
is reduced (e.g., in the eyes of a modelled audience) by a statement ($\alpha_4$) questioning the reliability of CO$_2$ measurements, without being evidenced (there is no support for $d$). %In real-world scenarios like this, we often encounter incomplete information – some temperature data might be missing, CO$_2$ measurements might be partial, or the exact relationship between these factors might not be fully understood. 
To date, existing approaches to GS in structured argumentation
%Traditional approaches may 
struggle to handle such incompleteness, potentially leading to oversimplified or incorrect conclusions, and thus poor modelling of debates.

To address %these limitations
this limitation, we introduce a novel methodology for %constructing %gradual semantics for structured argumentation frameworks 
GS in SGs that can naturally accommodate %\AR{contexts under 
complete %or 
and incomplete information. Our approach %, which we call the \textit{modular dialectical semantics}, 
is \emph{modular} in that it allows for different GS to be instantiated based on the requirements of a given task and context. 
Specifically, our methodology for GS can be instantiated with
%Then, a core tenet of our approach is the selection of an 
any existing GS for %abstract argumentation
QBAFs. %which 
This existing GS is used to dialectically evaluate the literals in a premise, based on their attacking and supporting evidence, before these ``intermediate'' evaluations are aggregated (depending on the statement's construction) to give a final strength for the statement.

Thus, our methodology combines some of the expressiveness of structured argumentation with the flexibility of GS for QBAFs, while accounting for the practical reality of incomplete knowledge.
Our contributions are as follows:
\begin{itemize}
    \item We introduce a novel methodology %, the \textit{modular dialectical semantics}, 
    for constructing %gradual semantics in structured argumentation frameworks 
    GS in SGs with incomplete information, presenting two %specific 
    instantiations based on different existing GS for QBAFs.
    \item We %define several properties for 
    present a collection of novel and existing (by Jedwabny \emph{et al.}~\shortcite{Jedwabny_20}) properties for GS in SGs and discuss their suitability in different settings. %both in general and in a particular application context of multi-agent model reconciliation. %adaptations of 
     %ones specific to our framework
    %
    \item Using these properties, we conduct a comprehensive theoretical analysis of our two proposed instantiations, two %different 
    instantiations 
     of the T-norm semantics~\cite{Jedwabny_20} and two existing GS for QBAFs, discussing the strengths and limitations of each approach. 
    %, again both in general and within our example context.
    %
    %\item We illustrate the practical applications of our approach through a case study in a model reconciliation setting~\cite{??}, and discuss its potential in other domains.
    %
\end{itemize}

\section{Related Work}
\label{sec:related}

%\delete{Abstract} 
GS have received significant attention within abstract argumentation in recent years, given their professed alignment with human reasoning~\cite{Vesic_22,Tarle_22%,Rago_23
} and applicability to real-world settings~\cite{Potyka_21,Rago_21%,Potyka_24
}. 
This has given rise to a rich field of research, with a broad repertoire of GS~\cite{Gonzalez_21,Yun_21,Wang_24}, in-depth analyses of GS' behaviour~\cite{Delobelle_19,Oren_22,%Yu_23,Baroni_19,
Yin_23,Kampik_24}, and a number of open-source computational toolsets~\cite{Potyka_22,Alarcon_24}.
One issue in this area is that many abstract GS, e.g., those we use \cite{Rago_16,Potyka_18}, do not satisfy the uniqueness property in assigning strengths in cyclic graphs due to their recursive nature \cite{Gabbay_15,Anaissy_24}. 
This must be considered if our acyclicity restriction is relaxed in future work.

%With regards to other approaches to work related to GS in structured argumentation, our approach differs from existing work as follows. 
There have been other works related to GS in structured argumentation, differing from our work as follows. Within ASPIC$^+$, Spaans~\shortcite{Spaans_21} introduced \emph{initial strength functions}, with corresponding properties, though they use weights on rules and do not consider explicit support relations. 
Further, Prakken~\shortcite{Prakken_24} gives a formal model of argument strength across graphs and deploys it. %However, the main 
A major difference between this work and our approach is our leveraging of existing bipolar abstract GS for handling incomplete information. We leave to future work the assessment of this approach with respect to our proposed properties. 
Also for ASPIC$^+$, extension-based semantics have been used for cases with incomplete information \cite{Odekerken_23}. %\delete{but does not consider the support relation and thus cannot be studied with respect to many of our properties.} 

%\note{Good spot, we will correct our comments based on this alternate notion of support. We will emphasise that the main difference between Prakken (2024) and our approach is our handling of incomplete information and leveraging of bipolar abstract GS, and that an assessment of the former is future work (rather than stating that it is incompatible).}

%a recent study by Prakken~\shortcite{Prakken_24} gives a formal model of argument strength and deploys it within ASPIC$^+$. However, this approach does not consider the support relation and is thus incompatible with a number of our properties.
Meanwhile, Heyninck \emph{et al.}~\shortcite{Heyninck_23} compare properties of GS with those concerning culpability when applied to abstract argumentation frameworks extracted from deductive argumentation, whereas we lift abstract GS to the structured level. 
%Relatedly, 
Skiba \emph{et al.}~\shortcite{Skiba_23} assess a family of GS %in
for ranking arguments in structured argumentation, specifically ABA. 
Both of these works %studies focus on analysing the satisfaction of properties by ranking semantics for AFs, i.e., without 
% ignore
do not consider the support relation.
%Meanwhile
Finally, Amgoud and Ben-Naim~\shortcite{Amgoud_15} give a method for ranking logic-based argumentation frameworks, but do not consider an explicit strength or a support relation.

Finally, \emph{probabilistic argumentation}~\cite{Kohlas_03,Dung_10,Thimm_12,Hunter_13,%Hunter_14,
Gabbay_15_prob,Polberg_17,%Hunter_17,
Hunter_20,Fazzinga_18,Thimm_18,Mantadelis_20}, %where uncertainty over arguments is modelled. 
%However,
%but in all of these works, either support relations are not considered or strengths are associated with probabilities and/or are not assigned to individual arguments. 
is a related but distinct subfield where strengths, which may be over (sets of) arguments, relations or both, must adhere to strict probabilistic principles, whereas GS may not.

\section{Preliminaries}
\label{sec:prelim}

We consider a propositional language $\mathcal{L}$ comprising a finite set of atoms, including a special atom $\top$ to represent truth. A \textit{literal} is an atom $\phi$ or its negation $\neg \phi$. Note that $\neg (\neg \phi) \equiv \phi.$ 
% A \textit{conjunction} is an expression constructed from literals using connective $\wedge$. 
A \textit{compound} is an expression, limited in this paper to be constructed from literals using %connectives 
the connective $\wedge$% and $\vee$
.

A \textit{statement} is defined as a %tuple $\alpha = \langle \Phi,\Psi \rangle$
pair $\langle \Phi,\Psi \rangle$, where $\Phi \in \mathcal{L}$ is either $\top$ or a logically consistent compound and $\Psi \in \mathcal{L}$ is a literal. We refer to $\Phi$ as the \emph{premise} %of the statement 
and to $\Psi$ as the \emph{claim} of statement $\alpha = \langle \Phi,\Psi \rangle$. We define $Prem(\alpha) = \{ x \: | \: x \in \Phi \} \setminus \{\top\}$ as the set of literals %comprising 
in the premise of %the 
statement $\alpha$. Note that if
the premise of a statement $\alpha$ is $\top$ then $Prem(\alpha) = \emptyset$.

%%%%%%%%%%%%

For simplicity and ease of presentation, %in this paper, 
we restrict %our
attention to statements %consisting of 
with conjunctive premises, similarly to %Jedwabny \emph{et al.}~\shortcite{Jedwabny_20}
\cite{Jedwabny_20}. 
 In §\ref{ssec:methodology}, we will define our methodology generally, before instantiating it for these restricted statements, pointing towards other instantiations which could be used for more complex compounds, such as those in conjunctive normal form. %\note{I think we should remove this sentence as we don;t actually consider this in this paper. AR: hmm maybe but I mention something before Example 2 (I know it's weak)}}
% \note{more justification here, say that we talk about extensions in section whatever}

%Then, a $\alpha_1$ \textit{attacks} (respectively, \textit{supports}) a statement $\alpha_2$ if $\alpha_1$ provides justification against (respectively, for) the premise of $\alpha_2$.

A statement \textit{attacks} or \textit{supports} another if the former %provides 
gives justification against or for the latter's premise:

\begin{definition}%[Attack and Support]
Let $\alpha_1 = \langle \Phi_1, \Psi_1 \rangle$ and $\alpha_2 = \langle \Phi_2, \Psi_2 \rangle$ be two statements. We say that: 
\begin{itemize}

\item$\alpha_1$ \textit{supports} $\alpha_2$ iff %$ \Psi_1 \in Prem(\alpha_2)$
$\exists x \in Prem(\alpha_2)$ s.t. $\Psi_1 \equiv x$; 
\item $\alpha_1$ \textit{attacks} $\alpha_2$ iff %$\neg\Psi_1 \in Prem(\alpha_2)$
$\exists x \in Prem(\alpha_2)$ s.t. $\Psi_1 \equiv \neg x$.
\end{itemize}
\label{def:supp-att}
\end{definition}

We can then generate %a statement graph (SG) 
an SG structured with respect to the support and attack relations between statements, as in the works by Hecham \emph{et al.}~\shortcite{hecham2018flexible} and Jedwabny \emph{et al.}~\shortcite{Jedwabny_20}.\footnote{Note that we chose to use SGs similarly to Hecham \emph{et al.}~\shortcite{hecham2018flexible} and Jedwabny \emph{et al.}~\shortcite{Jedwabny_20} in order to allow for direct comparisons with the gradual semantics in the latter work. Their focus on attacks and
supports on premises, rather than conclusions, also aligns
with other works in the literature, e.g., ABA.}
However, differently to these works, we include the quantitative %weighting 
weights of each statement within the SG, rather than applying them separately.
This reformulation allows for direct comparisons with %QBAFs and their
properties of QBAFs~\cite{Baroni_19}, which provided inspiration for some of our properties.

% [Statement Graph]
\begin{definition}
    An \emph{SG} is a quadruple $\langle \Args, \Atts, \Supps, \BS \rangle$ s.t.: 
        \begin{itemize}
         
        \item $\Args$ is a set of statements; 
        \item $\Atts \!\subseteq\! \Args \!\times\! \Args$ is %an \emph{attack} relation such that 
        s.t. $\forall \alpha_1, \!\alpha_2 \!\in\! \Args$, $(\alpha_1, \alpha_2) \!\in\! \Atts$ iff $\alpha_1$ \textit{attacks} $\alpha_2$; 
        \item $\Supps \!\subseteq\! \Args \!\times\! \Args$ is %a \emph{support} relation such that 
        s.t. $\forall \alpha_1, \!\alpha_2 \!\in\! \Args$, $(\alpha_1, \!\alpha_2) \!\in\! \Supps$ iff $\alpha_1$ \textit{supports} $\alpha_2$; 
        \item $\BS: \Args \rightarrow [0,1]$ is a \emph{weight} over the statements. 
       \end{itemize}
    For any $\alpha_1 \in \Args$, we call $\Atts(\alpha_1) = \{ \alpha_2 \in \Args | (\alpha_2, \alpha_1) \in \Atts \}$ $\alpha_1$'s \emph{attackers} and $\Supps(\alpha_1) =\{ \alpha_2 \in \Args | (\alpha_2, \alpha_1) \in \Supps \}$ $\alpha_1$'s \emph{supporters}.
    For any $\alpha_x, \alpha_y \in \Args$, let a \emph{path} from $\alpha_x$ to $\alpha_y$ via $\Rels \subseteq \Atts \cup \Supps$ be $(\alpha_0,\alpha_1), \ldots, (\alpha_{n-1}, \alpha_{n})$ for some $n>0$, where $\alpha_0 \!=\! \alpha_x$, $\alpha_n \!=\! \alpha_y$ and, for any $1 \leq i \leq n$, $(\alpha_{i-1}, \alpha_{i}) \in \Rels$.
    % We say that a SG $\langle \Args, \Atts, \Supps, \BS \rangle$ \emph{represents} some knowledge base $\KB = \langle \mathcal{R}_f, \mathcal{R}_n \rangle$ if $\forall r \in \mathcal{R}_f \cup \mathcal{R}_n$, $\exists \alpha \in \Args$, where $\alpha = \langle \Phi, \Psi \rangle$, such that $\Phi = Body(r)$ and $\Psi = Head(r)$.
    \label{def:graph}
\end{definition}

\begin{corollary}
\label{cor:disjointness}
    \!Given a SG $%\graph \!\!=\!\! 
    \langle \Args,\! \Atts,\! \Supps,\! \BS \rangle$,\! $\Atts \cap \Supps \!=\! \emptyset$.\footnote{The proofs for all theoretical results are given in Appendix \ref{appendix:proofs}.}
\end{corollary}

An SG thus gives a restricted form of structured argumentation framework with conjunctive premises, undercutting attacks, deductive/evidential supports and weights. Note that weights may be fixed, i.e., identical for all arguments which are considered a-priori equal as in classical argumentation frameworks, or variable, representing, e.g., different attributes such as source authority, premise credibility, or goal importance.\footnote{While methods for learning these weights represent an important research direction, e.g., as in \cite{Rago_25} within the setting of review aggregation, they fall outside the scope of this work.} We envisage a usage of weights similar to that in \cite{Alsinet_08} within the realm of \emph{defeasible logic programming} \cite{Garcia_04}, where weighted certainties on arguments (including facts) are used to introduce \emph{possibilistic argumentation.} Also, we focus on acyclic graphs, following several works on GS \cite{Rago_16,Tarle_22%,Rago_23
}, leaving cyclic graphs to future work%, as in. Though cyclic graphs have been partially addressed in previous work \cite{Potyka_18}, extending our methods to handle cyclic graphs remains a direction for future research
.

For the remainder of the paper, we assume as given a generic SG $\graph = \langle \Args, \Atts, \Supps, \BS \rangle$.
\begin{definition}
\label{def:semantics}
    %Given a weighted SG $\graph = \langle \Args, \Atts, \Supps, \BS \rangle$, a 
    A \emph{GS} $\SF: \Args 
  \rightarrow [0,1]$ assigns each statement $\alpha \in \Args$ a \emph{strength} $\SF^\graph
    (\alpha) \in [0,1]$.\footnote{With an abuse of notation, when referring to GS generally, we drop the $\graph$ superscript for clarity.}
\end{definition}

In essence, GS offer a principled approach to \textit{evaluating} statements quantitatively. %, i.e., for assigning arguments an acceptance on a gradual scale. %In this paper, w
%\FT{Note that, differently from standard definitions of GSs for QBAFs, we include explicitly the SG. }
%\todo{AR: moved here, not sure why we had it later as the T-norm def refers to structured GS}
We refer informally throughout the paper to any GS that %considers 
takes into account the premises of any statement as a \emph{structured GS}, and any that does not as an \emph{abstract GS}. 
It should be noted that while abstract GS can be applied to SGs, structured GS cannot be applied to frameworks using abstract argumentation such as QBAFs, since the statements therein are abstract entities.

%\note{introduce bottom-strength and top-strength and propagate throughout (i.e., instead of accepted and rejected). also talk about other notions of acceptance, e.g., the threshold as in the De Tarle paper maybe}

We refer to any statement $\alpha \in \Args$ with $\SF^\graph(\alpha) = 1$ as having/being \emph{top-strength}, and $\SF^\graph(\alpha) = 0$ as having/being \emph{bottom-strength}. %\note{SV: Should we talk about acceptance/rejection wrt a threshold? \AR{I think so}}
We note here that other notions of acceptance and rejection, e.g., with thresholds over the strength range as in \cite{Tarle_22,Rago_23}, would be possible in our setting, but we leave the study of this to future work.

Next, we recall two notions %for the GS 
from~\cite{Jedwabny_20}.\footnote{We omit the acyclicity condition since we use acyclic graphs.} 
The first gives the notion indicating when an argument's premises is grounded in facts, while the second defines operators used in the GS.

% [Complete Support Tree]
\begin{definition}
\label{def:CST}
    A \emph{complete support tree} (CST) for some $\alpha_1 \in \Args$ is a set of statements $T \subset \Args$ s.t.: 
    \begin{itemize}
        \item $\alpha_1 \in T$;
        \item $\forall \alpha_2 \!\in\! T$, $\forall x \!\in\! Prem(\alpha_2)$, $\exists \alpha_3 \!=\! \langle \Phi_3, \Psi_3 \rangle \!\in\! T$ s.t. $\Psi_3 \equiv x$; 
        \item $\nexists \alpha_2, \alpha_3 \in T$ s.t. $(\alpha_2,\alpha_3) \in \Atts$; 
        %
        % \item \delete{$\nexists \alpha_2, \alpha_3 \in T$ such that there is a directed cycle of support edges between $\alpha_2$ and $\alpha_3$;}
        %
        \item $\nexists T' \!\!\subset\!\! T$ s.t. $T'$ is a CST for $\alpha_1$ (minimal wrt set inclusion). 
    \end{itemize}
    We denote with $\CST^\graph(\alpha_1)$ the set of CSTs for $\alpha_1$ wrt $\graph$. Given two CSTs $T\in \CST^\graph(\alpha_1)$ and $T' \in \CST^\graph(\alpha_2)$ for $\alpha_1, \alpha_2 \in \Args$, respectively, $T$ attacks $T'$ iff $\exists \alpha_3 \in T'$ s.t. $(\alpha_1,\alpha_3) \in \Atts$.
\end{definition}
% [De Morgan Triple]
\begin{definition}
A triple $(\otimes, \oplus, \neg)$, where $\otimes$ is a T-norm, $\oplus$ is a T-conorm, and $\neg$ is a negation, is called a De Morgan triple iff $\forall v_1,v_2 \in [0,1]$, $\otimes (v_1,v_2) = \neg (\oplus (\neg (v_1), \neg (v_2))$, and $\oplus (v_1,v_2) = \neg (\otimes (\neg (v_1), \neg (v_2))$.
\end{definition}

Then, Jedwabny \emph{et al.}~\shortcite{Jedwabny_20} introduce the following GS.
% [T-norm Semantics]
\begin{definition}
\label{def:Tnorm}
    Given %a weighted SG $\graph = \langle \Args, \Atts, \Supps, \BS \rangle$ 
   a De Morgan triple $(\otimes, \oplus, \neg)$, the \emph{T-norm semantics} $\SF_T$ is a structured GS s.t. for any $\alpha \in \Args$, where
    %\begin{align}
        %
        $\SF^\graph_T(\alpha) = \bigoplus_{T \in \CST^\graph(\alpha)} \mathcal{O}(T)$, % \nonumber \\
        $\mathcal{O}(T) = \mathcal{I}(T) \otimes \neg \bigoplus_{T' \in \CST^\graph(\alpha' \in \Args)\text{, } T' \text{ attacks } T} \mathcal{I}(T')$ and %\nonumber \\
        $\mathcal{I}(T) = \bigotimes_{\alpha \in T} \tau(\alpha)$. %\nonumber
        %
    % \end{align}
\end{definition}

While the T-norm semantics is modular wrt the De Morgan triple, the methodology is based on a statement's CSTs, and thus it can be seen that it inherently relies on complete information, i.e., a statement's grounding in facts.

In this paper, we instantiate two T-norm semantics. For the first, which we name the \emph{T-norm-p} semantics, denoted by $\SF_{T_p}$, we let, as in the illustrations by Jedwabny \emph{et al.}~\shortcite{Jedwabny_20}: 
$\otimes$ be the product, i.e., $\otimes(v_1,v_2) = v_1 \times v_2$;
$\oplus$ be the probabilistic sum, i.e., $\oplus(v_1,v_2) = v_1 + v_2 - (v_1 \times v_2)$; and 
$\neg$ be standard negation, i.e., $\neg(v_1) = 1 - v_1$.
For the second, which we name the \emph{T-norm-m} semantics, denoted by $\SF_{T_m}$, we let: 
$\otimes$ be the minimum, i.e., $\otimes(v_1,v_2) = \min(v_1,v_2)$;
$\oplus$ be the maximum, i.e., $\oplus(v_1,v_2) = \max(v_1,v_2)$; and 
$\neg$ be standard negation as for the T-norm-p semantics.
%We plan to assess other instantiations of the T-norm semantics as future work.

We now recap two abstract GS from the literature, as defined by Rago \emph{et al.}~\shortcite{Rago_16} and Potyka~\shortcite{Potyka_18}, respectively.

% [DF-QuAD Semantics]
\begin{definition}
\label{def:DFQuAD}
    %Given a weighted SG $\graph = \langle \Args, \Atts, \Supps, \BS \rangle$, t
    The \emph{DF-QuAD semantics} $\SF_D$ is an abstract GS s.t. for any $\alpha \in \Args$, 
    %\begin{align}
        $\SF^\graph_D(\alpha) = c(\BS(\alpha),\Sigma(\SF^\graph_D( \Atts(\alpha))),\Sigma(\SF^\graph_D(\Supps(\alpha))))$ %\nonumber 
    % \end{align}
    where, for any $S \subseteq \Args$, $\SF^\graph_D(S)=(\SF^\graph_D(\alpha_1),\ldots,\SF^\graph_D(\alpha_k))$ for $(\alpha_1,\ldots,\alpha_k)$, an arbitrary permutation of $S$, and: 
    \begin{itemize}
        \item $\Sigma$ is s.t. $\Sigma(())=0$, where $()$ is an empty sequence, and, for $v_1,\ldots,v_n \in [0,1]$ ($n \geq 1$), 
        if $n=1$, then $\Sigma((v_1))=v_1$; if $n=2$, then $\Sigma((v_1,v_2))= v_1 + v_2 - v_1\cdot v_2$; and 
        if $n>2$, then $\Sigma((v_1,\ldots,v_n)) = \Sigma (\Sigma((v_1,\ldots, v_{n-1})),v_n)$; 
        \item $c$ is s.t., for $v^0,v^-,v^+ \in [0,1]$,
        if $v^-\geq v^+$, then $c(v^0,v^-,v^+)=v^0-v^0\cdot| v^+ - v^-|$ and
        if $v^-< v^+$, then $c(v^0,v^-,v^+)=v^0+(1-v^0)\cdot| v^+ - v^-|$.
    \end{itemize}
\end{definition}

% [QEM Semantics]
\begin{definition}
\label{def:QEM}
    %Given a weighted SG $\graph = \langle \Args, \Atts, \Supps, \BS \rangle$, t
    The \emph{QEM semantics}\footnote{We %describe a simplification of the original algorithm for 
    define a simplified GS for the case of acyclic graphs.} $\SF_Q$ is an abstract GS s.t. for any $\alpha \in \Args$,
    %\begin{align}
        $\SF^\graph_Q(\alpha) = \BS(\alpha) + (1 - \BS(\alpha)) \cdot h(E_\alpha) -  \BS(\alpha) \cdot h(-E_\alpha)$ %\nonumber 
    %\end{align}
    where $E_\alpha = \sum_{\alpha_1 \in \Supps(\alpha)}{\SF^\graph_Q(\alpha_1)} - \sum_{\alpha_2 \in \Atts(\alpha)}{\SF^\graph_Q(\alpha_2)}$ and for all $v \in \mathbb{R}$, $h(v) = \frac{\max\{v,0\}^2}{1+\max\{v,0\}^2}$.
\end{definition}

While these abstract GS do not take into account the structure of statements, they do consider a weight on each statement, which gives an intrinsic strength which may be used as a starting point for calculating statement strengths, e.g., in the case of incomplete information.

%%%%%%%%%%%%%%%%%%%%%%%%%%%%%%%%%%%%%%%%%
%%%%%%%%%%%%%%%%%%%%%%%%%%%%%%%%%%%%%%%%%

\section{%A Methodology for 
Gradual Semantics for Statement Graphs}
\label{sec:main}

We now introduce a novel methodology for defining structured GS, leveraging abstract GS, to evaluate the strength of statements in SGs  (§\ref{ssec:methodology}). We then %, for our setting, 
redefine and discuss the suitability of existing (redefined for this setting) and novel properties for GS in SGs (§\ref{ssec:properties}), %which motivates our introduction of novel properties (§\ref{ssec:novelproperties}), 
before we evaluate our approach with %both sets of 
the properties, comparing against existing structured and abstract GS (§\ref{ssec:theory}). These contributions will be discussed with a focus of the following notion.

% [Completeness]

\begin{definition}
\label{def:completeness}
    For any $\alpha_1 \in \Args$, we say $\alpha_1$ is: 
    \begin{itemize}
        % \item $\alpha_1$ is \emph{complete} iff $\forall \alpha_2 \in \Args$ such that $\alpha_2$ supports $\alpha_1$, there exists a CST for $\alpha_2$;
        %\item $\alpha_1$ is 
        \item \emph{complete} iff $\forall \alpha_2 \in \Args$, where $\alpha_2 = \alpha_1$ or there exists a path from $\alpha_2$ to $\alpha_1$ %\SV{s.t. $\alpha_2$ supports $\alpha_1$} 
        via $\Supps$, $\CST^\graph(\alpha_2) \neq \emptyset$; 
        %$\exists$ at least one CST for $\alpha_1$ and $\forall \alpha_2 \in \Supps(\alpha_1)$, there exists at least one CST for ...
        %
        %\item $\alpha_1$ is 
        \item \emph{partially-complete} iff it is not complete and $\CST^\graph(\alpha_1) \neq \emptyset$; 
        %
        % \item $\alpha_1$ is 
        \item \emph{incomplete} otherwise, i.e., iff $\CST^\graph(\alpha_1) = \emptyset$.
    \end{itemize}
\end{definition}

Note that partially-complete and incomplete statements are in the same spirit as \emph{potential arguments} in ABA \cite{potentialToni13}, but the latter are used for computational purposes, as steps towards complete statements, rather than as representations in their own right, as we do. Further investigation, e.g. on the equivalent in ASPIC$^+$, is left to future work.

% \AR{
\begin{example}
\label{ex:completeness}
Consider the SG in Figure \ref{fig:completeness}, in which 
$\Args = \{ \alpha'_1, \alpha'_2, \alpha'_3 \}$, 
$\alpha'_1 = \langle b, a \rangle$, 
$\alpha'_2 = \langle c, b \rangle$ and
$\alpha'_3 = \langle d, b \rangle$, giving
$\Supps = \{ (\alpha'_2, \alpha'_1), (\alpha'_3, \alpha'_1) \}$ and $\Atts = \emptyset$ (the $\BS$ values are irrelevant here).
In this first case, $\alpha'_1$ incomplete as it has no CSTs.
If we add the statement $\alpha'_4 = \langle \top, c \rangle$, leading to $\Supps = \{ (\alpha'_2, \alpha'_1), (\alpha'_3, \alpha'_1), (\alpha'_4, \alpha'_2) \}$, $\alpha'_1$ becomes partially-complete as it has one CST, but $\alpha'_3$, which has a path to $\alpha'_1$, does not have a CST.
Finally, if we add the statement $\alpha'_5 = \langle \top, d \rangle$, leading to $\Supps = \{ (\alpha'_2, \alpha'_1), (\alpha'_3, \alpha'_1), (\alpha'_4, \alpha'_2), (\alpha'_5, \alpha'_3) \}$, $\alpha'_1$ becomes complete as all statements with a path to it have CSTs.
Crucial here is the fact that in the incomplete case the T-norm semantics ignores the (potentially useful) information from both $\alpha'_2$ and $\alpha'_3$, regardless of $\BS$, while in the partially-complete case it ignores that of $\alpha'_3$.
\end{example}
% }

\begin{figure}[t]
    \centering
    \includegraphics[width=0.65\linewidth]{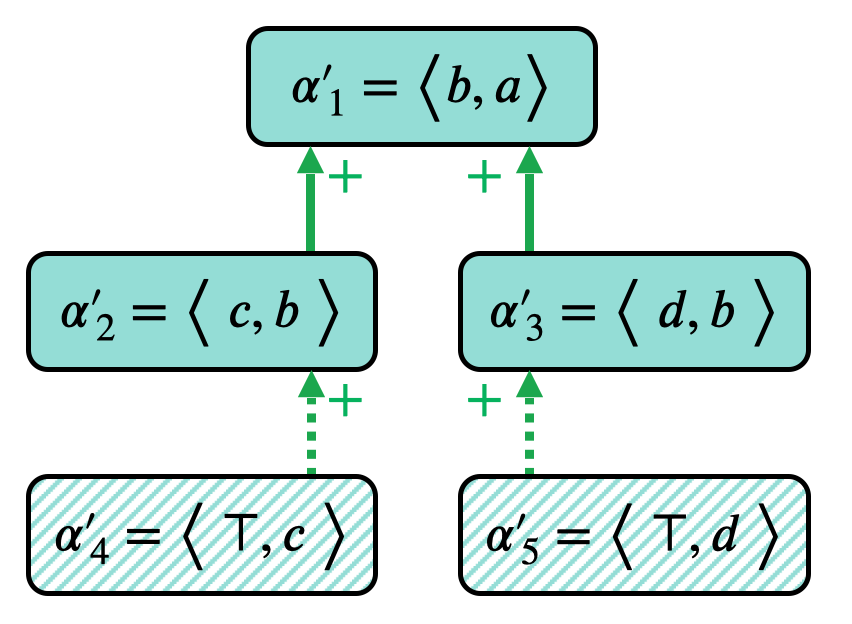}
    \caption{SG from Example \ref{ex:completeness} with initial (added) statements shown in solid (striped, respectively) turquoise, no attacks, initial (added) supports as solid (dashed, respectively) green arrows and some arbitrary $\BS$ (not shown here). The statement $\alpha'_1 \in \Args$ is incomplete when $\Args = \{ \alpha'_1, \alpha'_2, \alpha'_3 \}$, but is partially-complete when $\alpha'_4$ (or $\alpha'_5$) is added and complete when both $\alpha'_4$ and $\alpha'_5$ are added.}
    \label{fig:completeness}
\end{figure}

To solve this issue, we target a methodology which behaves well in all cases, i.e., considering all available information.

\begin{figure*}[h]
    \centering
    \includegraphics[width=\textwidth]{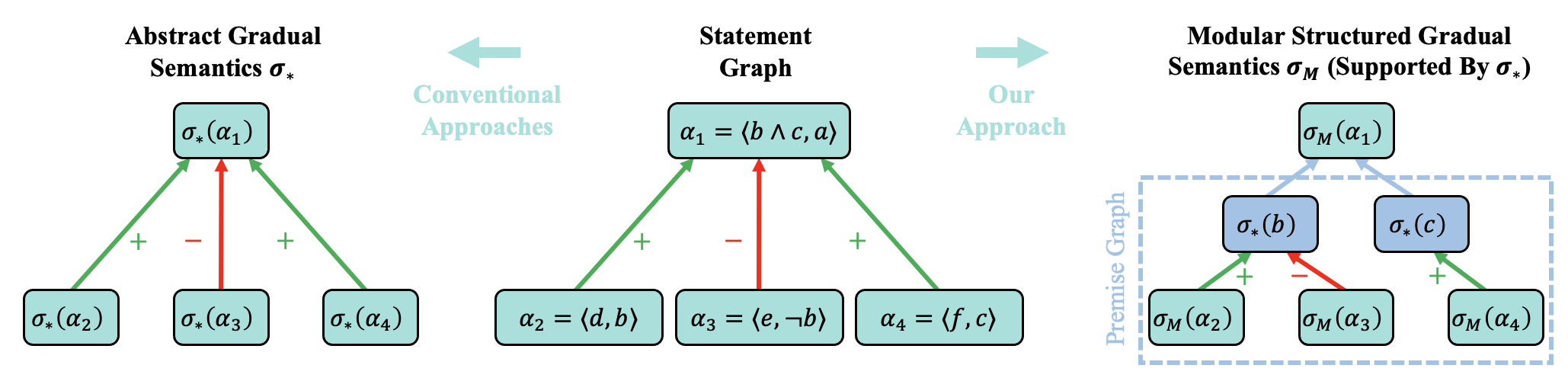}
    \caption{The conventional way (left) in which abstract GS ($\SF_*$) are used for SGs (centre), with a strength calculated for $\alpha_1$ using the strengths of its supporters ($\alpha_2$ and $\alpha_4$) and attacker ($\alpha_3$), and our approach (right), modular structured GS ($\SF_M$), which calculate the strength, using $\SF_*$, of each literal in $\alpha_1$'s premises ($b$ and $c$) within the premise graph of $\alpha_1$, before $\SF_M$ aggregates these strengths in a premise aggregation function to calculate the strength of $\alpha_1$. %\note{align with Figure \ref{fig:SG} stylistically somehow}
    }
    \label{fig:mds}
\end{figure*}

\subsection{Methodology}
\label{ssec:methodology}

First, we introduce the following notion.

% \AR{
\begin{definition}
    A \emph{premise graph} for some $\alpha_1 \in \Args$ is an SG $\graph^1 \!\!=\!\! \langle \Args^1\!, \!\Atts^1\!, \!\Supps^1\!, \!\BS^1\! \rangle$ s.t. $\Args^1 \!\!=\! Prem(\alpha_1) \!\cup\! \Atts(\alpha_1) \!\cup\! \Supps(\alpha_1)$ and:

        \begin{itemize}
         
        \item $\Atts^1 = \{ ( \langle \Phi_2, \Psi_2 \rangle, x) \in \Atts(\alpha_1) \times Prem(\alpha_1) \: | \: %, \langle \Phi_2, \Psi_2 \rangle \in \Atts(\alpha_1)
        \neg x \equiv \Psi_2 \}$; 
        \item $\Supps^1 = \{ ( \langle \Phi_2, \Psi_2 \rangle, x) \in \Supps(\alpha_1) \times Prem(\alpha_1) \: | \: %, \langle \Phi_2, \Psi_2 \rangle \in \Atts(\alpha_1)
        x \equiv \Psi_2 \}$.

    \end{itemize}
\end{definition}
% }

A statement's premise graph contains a new set of statements representing the statement's premises, attackers and supporters from the original SG, as shown in the dashed blue box on the right of 
Figure \ref{fig:mds}% for $\alpha_1$ in the statement graph on the left
.
Note that we do not specify $\BS^1$ at this point; we will do so in the specific instantiation of our methodology.
In introducing a premise graph, we allow for the literals in statements' premises to be evaluated by abstract gradual semantics, a core tenet of our approach.

Our methodology, illustrated in Figure \ref{fig:mds}, %for structured GS 
is as follows.\footnote{Given the modular nature of our approach, we leave a study of the complexities of the resulting GS to future work.}

\begin{definition}
\label{def:modular}
    Given an abstract GS $\SF_*$, \emph{a modular structured GS supported by} $\SF_*$ is a structured GS $\SF_{M}$ s.t. for any $\alpha_1 \in \Args$, where $\graph^1$ is the premise graph of $\graph$ for $\alpha_1$:
    \begin{align}
        \SF_{M}^\graph(\alpha_1) = \odot(\{ \SF_*^{{\graph^1}}(x) \: | \: x \in Prem(\alpha_1) \}) \nonumber
    \end{align}
    where $\odot: [0,1]^* \rightarrow [0,1]$ is a \emph{premise aggregation function}. %\delete{and %(with an abuse of notation) 
\end{definition}

The modular structured GS thus dialectically evaluates the evidence supporting or attacking each literal in a statement's premise via an abstract GS using the premise graph, rather than based on the completeness of its support as in~\cite{Jedwabny_20}. %Once the literals have such a 
These dialectical evaluations are then used to evaluate the statement with the premise aggregation function.
This permits the incorporation of incomplete information, since the evidence for or against premises is combined dialectically starting from a statement's weight, giving an approximation of the strength of a statement based on the \emph{available} evidence, giving desirable behaviour in all three cases in Definition \ref{def:completeness}.
%Statements' strengths are thus calculated recursively, as in abstract GS.
This modular methodology permits instantiations tailored to individual applications based on %chosen instantiation of %the supporting abstract GS, the definition of \AR{literal} weights and the premise aggregation function
choices in Definition \ref{def:modular}. 
We now define one such instantiation, tailored to our restricted SGs of Definition \ref{def:graph}, where premises are conjunctions, with its supporting abstract GS left unspecified.

\begin{definition}
\label{def:DC}
    Given %a weighted SG $\graph = \langle \Args, \Atts, \Supps, \BS \rangle$ and 
    an abstract GS $\SF_*$, the \emph{dialectical conjunction (DC) semantics} supported by $\SF_*$, denoted by $\SF_{\wedge_*}$, is a modular structured GS s.t., for any $\alpha_1 \in \Args$, if $Prem(\alpha_1) = \emptyset$, then $\SF_{\wedge_*}^\graph(\alpha_1) = \BS(\alpha_1)$; otherwise: 
    \begin{align}
        \SF_{\wedge_*}^\graph(\alpha_1) = \bigtimes_{x \in Prem(\alpha_1)} \SF_*^{\graph^1}(x) \nonumber
    \end{align}
    where $\BS^1(x) = \sqrt[n]{\BS(\alpha_1)}$, $n = |Prem(\alpha_1)|$ and $\BS^1(\alpha_2) = \SF_{\wedge_*}^\graph(\alpha_2)$ $\forall \alpha_2 \in \Args^1 \setminus Prem(\alpha_1)$.
\end{definition}

\begin{corollary}
 
    For any abstract GS $\SF_*$ that satisfies existence and uniqueness, %i.e., for any $\alpha \in \Args$ where $\SF^\graph_{*}(\alpha)$ always exists and is unique, 
    $\SF_{\wedge_*}$ also satisfies existence and uniqueness.
    
\end{corollary}

Note that conjunctive premises represent a challenge for handling incomplete information, as the failure of any single premise invalidates the entire statement. For that matter, the DC semantics has been designed such that certain desirable theoretical properties will be satisfied, %e.g., the simple multiplication operator and $\BS^1$ ensures that the argument's weight remains the starting point for its strength. We will elaborate on this 
as we will see in §\ref{ssec:properties}.  Importantly, the DC semantics represents a natural starting point for our methodology because conjunctive premises are typically abound in real-world argumentative reasoning. For instance, when people construct arguments, they typically combine multiple pieces of evidence conjunctively to support their claims. This pattern is evident in scientific argumentation, legal reasoning and everyday discourse, where people naturally aggregate supporting facts with ``and'' connectives \cite{mercier2011humans}.
In future work, modular structured GS could be instantiated for different structures of SGs, e.g., for SGs with disjunctions as premises, a probabilistic sum could be used for the premise aggregation function. This would align with the intuition that multiple independent supporting lines of reasoning may strengthen a statement's overall strength. For other logical structures such as statements in %\todo{AR: we mention it earlier like this}
conjunctive normal form, a hierarchical aggregation approach could be employed% as explored
, i.e., first combining literals within clauses using disjunctive operators, then aggregating clauses conjunctively.

\begin{example}
\label{ex:2}
Consider the SG from Figure~\ref{fig:SG}, where $\Supps = \{ {(\alpha_2, \alpha_1), (\alpha_3, \alpha_1) } \}$ and $\Atts = \{ {(\alpha_4, \alpha_1)} \}$. 
% Here, $\alpha_1 = \langle a \wedge b, d \rangle$ represents the statement ``CO$_2$ emissions are rising and global temperatures are increasing, therefore climate change is human-caused.''
Given the DC semantics supported by DF-QuAD, the strength of statement $\alpha_1$ is calculated as follows:
$\SF_{\wedge_D}^\graph(\alpha_1) = \SF_{D}^{\graph^1} (a) \times \SF_{D}^{\graph^1}(b) = 0.916 \times 0.957 = 0.877$, where $\BS^1(a) = \BS^1(b) = \sqrt[2]{0.8} = 0.894$. In contrast, the T-norm-p semantics produces a strength of: $\SF_{T_p}^\graph(\alpha_1) = \tau(\alpha_1) \times \tau(\alpha_2) \times \tau(\alpha_3) = 0.8 \times 0.9 \times 0.6 = 0.432$.
% \note{Reviewer was confused so we answered this: \emph{Regarding Example 2, premise a has both an attacker, $\alpha_4$, and a supporter, $\alpha_2$, and premise b has a supporter, $\alpha_3$.} I think maybe the figure should help now}
Here, we see how the DC and T-norm-p semantics differ. The DC semantics considers the individual strengths of the premises of $\alpha_1$, while the T-norm-p semantics takes into account the weight of $\alpha_1$ and its direct supporters.
\end{example}

\subsection{Properties}
\label{ssec:properties}

We now discuss existing (from \cite{Jedwabny_20}\footnote{Jedwabny \emph{et al.}~\shortcite{Jedwabny_20} also defined properties that are specific to CSTs, which we reformulate and assess in Appendix \ref{appendix:props}.}) and novel properties for GS in SGs. 
Note that the properties apply to either all or exclusively structured GS as specified.

The first property \cite{Jedwabny_20} is a fundamental requirement for GS, stating that a statement's strength only depends on statements that are connected to it, and are thus relevant to it.

\begin{property}%[Directionality]
\label{prop:directionality}
    %Given a second  SG
    Given two SGs $\graph$ and $\graph' = \langle \Args \cup \{ \alpha_1 \}, \Atts', \Supps', \BS' \rangle$, where $\alpha_2 = \alpha_1$ or $\alpha_3 = \alpha_1$ $\forall (\alpha_2,\alpha_3) \in (\Atts' \cup \Supps') \setminus (\Atts \cup \Supps)$, a GS $\SF$ satisfies \emph{directionality} iff for any $\alpha_2 \in \Args$, where there exists no path from $\alpha_1$ to $\alpha_2$ via $\Atts \cup \Supps$ and $\BS'(\alpha_3) = \BS(\alpha_3)$ $\forall \alpha_3 \in \Args$, it holds that $\SF^{\graph'}(\alpha_2) = \SF^\graph(\alpha_2)$.
\end{property}

As in previous works~\cite{Amgoud_18,Amgoud_22}, we would advocate that 
%such a
this property is %desirable across 
satisfied in the vast majority of settings
since we would not expect %that 
information that is unrelated to a statement to have an effect on its strength.

Next, Jedwabny et al.~\shortcite{Jedwabny_20} introduced a property which requires that the length of a proof in the logic, i.e., a chain of reasoning, should not affect the strength of a statement.

\begin{property}%[Rewriting]
\label{prop:rewriting}
    A structured GS $\SF$ satisfies \emph{rewriting} iff for any $\alpha_1, \alpha_2, \alpha_3 \in \Args$, if: 
    %\squishlist
        %
        %\item 
        \begin{itemize}
        
        \item $\alpha_1 = \langle \Phi_1, \Psi_1 \rangle$, where $\Phi_1 \equiv x_1 \wedge \ldots \wedge x_n \wedge x_{n+1}$ and $\Psi_1 \equiv x$; 
        %
        %\item 
        \item $\alpha_2 = \langle \Phi_2, \Psi_2 \rangle$, where $\Phi_2 \equiv x_1 \wedge \ldots \wedge x_n$ and $\Psi_2 \equiv x'$; 
        %
        %\item 
        \item $\alpha_3 = \langle \Phi_3, \Psi_3 \rangle$, where $\Phi_3 \equiv x' \wedge x_{n+1}$ and $\Psi_3 \equiv x$; 
        %
        %\item 
        \item $\nexists \langle \Phi_4, \Psi_4 \rangle$ where $x' \in Prem(\langle \Phi_4, \Psi_4 \rangle)$ or $x' \equiv \Psi_4$; 
        % 
        %\item 
        \item $\BS(\alpha_1) = \BS(\alpha_3)$ and $\BS(\alpha_2) = 1$; 
   %\squishend
   \end{itemize}
    then $\SF^\graph(\alpha_1) = \SF^\graph(\alpha_3)$.
\end{property}

This property is indeed intuitive in a number of settings where we would prioritise the logical consistency over approximate evaluations of statements. However, in some settings, e.g., when modelling human discourse, it may be the case that the same logical conclusions, when reached over the course of multiple statements, are weaker, or even stronger, than those stated succinctly (e.g., see the study of \emph{priority} in \cite{Yin_24}). 
For example, given the statements instantiated in rewriting, consider the effect of adding a statement $\alpha_4 = \langle \Phi_4, \Psi_4 \rangle$, where $\Psi \equiv \neg x_1$. In some settings it may be %seems it would be 
desirable that the weakening effect of $\alpha_4$ on $\alpha_3$ is lesser than that on $\alpha_1$ due to the distance. 
We leave a formalisation of this (non-trivial) property to future work.
We would, however, %therefore 
argue that rewriting may not be \emph{universally} applicable, as illustrated below.

\begin{example}
\label{ex:3}

Continuing from Example~\ref{ex:2}, let us introduce two additional statements in the SG of Figure~\ref{fig:SG}:
%\squishlist
%\item 
$\alpha_5 = \langle a, e \rangle$ (``CO$_2$ emissions are rising, therefore extreme weather events are increasing''), and
%\item 
$\alpha_6 = \langle e \wedge b, c \rangle$ (``Extreme weather events and global temperatures are increasing, therefore climate change is human-caused'').
%\squishend
Let $\tau(\alpha_5) = 1$ and $\tau(\alpha_6) = 0.8$. 
% We now have: $\Supps = \{(\alpha_2, \alpha_1), (\alpha_3, \alpha_1), (\alpha_3, \alpha_6), (\alpha_2, \alpha_5), (\alpha_5, \alpha_6) \}$, and 
$\Atts = \{(\alpha_4, \alpha_1), (\alpha_4, \alpha_5)\}$.
% Observe that $\alpha_1$ and $\alpha_6$ both make the same claim $d$, i.e., ``Climate change is human-caused''.}
Under the DC semantics, we get $\SF_{\wedge_D}^\graph(\alpha_1) = 0.877$ and $\SF_{\wedge_D}^\graph(\alpha_6) = \SF_{D}^{\graph^1}(e) \times \SF_{D}^{\graph^1}(b) = 1 \times 0.957 = 0.957$. In contrast, under T-norm-p semantics, we get $\SF_{T_p}^{\graph}(\alpha_1) \!=\! \SF_{T_p}^\graph(\alpha_6) \!=\! 0.432$.
We can see that the DC semantics violate the rewriting property, as $\alpha_1$ and $\alpha_6$ have different strengths despite making the same claim through different reasoning paths. The T-norm-p semantics, however, is agnostic to the different lengths of the reasoning chain. 
\end{example}

Next, Jedwabny \emph{et al.}~\shortcite{Jedwabny_20} state that a statement without at least one supporter for each of its premises, i.e., with no evidence for any of them, is bottom-strength.

\begin{property}%[Provability]
\label{prop:provability} 
    A structured GS $\SF$ satisfies \emph{provability} iff for any $\alpha_1 \in \Args$, if $\exists x \in Prem(\alpha_1)$ and $\nexists \langle \Phi_2, \Psi_2 \rangle \in \Supps(\alpha_1)$ s.t. $x \equiv \Psi_2$, then $\SF^\graph(\alpha_1) = 0$.
\end{property}

%This is clearly a driving property of the T-norm semantics, where CSTs are required for statements to be considered accepted in any sense.
This property is clearly useful in cases under %perfect 
complete information, i.e., first case in Definition \ref{def:completeness}, but in the other two cases% with partially-complete or incomplete statements
, we posit that this property may be too demanding.
% as illustrated in the following example.

\begin{example}
\label{ex:4}
%Continuing 
In our running example, %let us focus on statement 
consider $\alpha_4 = \langle d, \neg a \rangle$. %Observe 
Note that, while $\alpha_4$ is an attacker of $\alpha_1$ and $\alpha_5$, it is incomplete because its premise $d$ does not have a support.
Under the T-norm-p semantics, the strength of $\alpha_4$ is $\SF^{\graph}_{T_p}(\alpha_4) = 0$. Consequently, it has no effect on the strengths of $\alpha_1$ and $\alpha_5$. In contrast, under the DC semantics, the strength of $\alpha_4$ defaults to its weight, i.e., $\SF^{\graph}_{\wedge_D}(\alpha_4) = \tau(\alpha_4) =  0.7$. This non-zero strength is then used when computing the strengths of $\alpha_1$ and $\alpha_5$, thus reducing their overall strength.
\end{example}

As an optional alternative, we propose a weaker property (trivially implied by provability), which makes the same demands but only when the statement's weight is zero.

\begin{property}%[Weak Provability]
\label{prop:weakprovability} 
    A structured GS $\SF$ satisfies \emph{weak provability} iff for any $\alpha_1 \in \Args$ such that $\BS(\alpha_1) =0$, if $\exists x \in Prem(\alpha_1)$ and $\nexists \langle \Phi_2, \Psi_2 \rangle \in \Supps(\alpha_1)$ s.t. $x \equiv \Psi_2$, then $\SF^\graph(\alpha_1) = 0$.
\end{property}

\begin{corollary}
\label{cor:weakprovability} 
   A structured GS $\SF$ which satisfies provability necessarily satisfies weak provability.
\end{corollary}

In Example \ref{ex:3}, weak provability would not require that $\alpha_4$'s strength is zero, allowing a form of ``trust'' in the statement's weight in the absence of complete information providing the logical reasoning to support its premises. This means the information is not lost in the partially-complete or incomplete cases, unless it is assigned the minimal %strength 
weight by the GS, which essentially indicates it cannot be trusted.

%\delete{Given the fact that the properties we have covered in §\ref{ssec:completeproperties} fail to govern the behaviour of GS under incomplete information in a number of cases, w}
We now propose %some 
more novel properties that offer
%should be considered alongside these existing properties as part of a collection that determines structured GS' applicability for given applications.
%Firstly, as part of a set of properties providing 
%an 
alternatives to provability, %we 
firstly considering the ``initial state'' of statements, i.e., 
that the strength of a statement with no attackers or supporters should be equal to its weight, as in properties defined for abstract GS, namely \emph{stability}~\cite{Amgoud_18} and \emph{balance}~\cite{Baroni_19}.
It is also related to \emph{logical void precedence} \cite{Heyninck_23}% for the setting of abstract argumentation frameworks
.

\begin{property}%[Stability]
\label{prop:stability}
    A GS $\SF$ satisfies \emph{stability} iff for any $\alpha \in \Args$, if $\Atts(\alpha) \!=\! \Supps(\alpha) \!=\! \emptyset$, then $\SF^\graph(\alpha) \!=\! \BS(\alpha)$.
\end{property}

The desirability of this property is clear: as demonstrated below, it gives %demonstrating 
an intuitive starting point for the strength of a statement, regardless of its completeness from Definition \ref{def:completeness}.

\begin{example}
Using statement $\alpha_4 = \langle d, \neg a \rangle$ again, note that $\Supps(\alpha_4) = \Atts(\alpha_4) = \emptyset$, meaning that $\alpha_4$ has no supporters or attackers. Under the DC semantics, $\SF^{\graph}_{\wedge_D}(\alpha_4) = \tau(\alpha_4) = 0.8$, whereas, under the T-norm-p semantics, $\SF^{\graph}_{T_p}(\alpha_4) = 0$.
The DC semantics respects the stability property, defaulting to the weight for $\alpha_4$'s strength. The T-norm-p semantics, however, rejects $\alpha_4$ due to its lack of a CST.
\label{ex:7}
\end{example}

Stability then works in tandem with the following properties such that the behaviour of semantics can be proven to be desirable iteratively over any state under any completeness.

First, we consider the effect of adding a statement that has been assigned a zero strength, i.e., it should not affect the strength of any statement it attacks or supports, as in \emph{neutrality}~\cite{Amgoud_18} for abstract GS.

\begin{property}%[Neutrality]
\label{prop:neutrality}
    %Given a second  SG
    Given two SGs $\graph$ and $\graph' \!=\! \langle \Args \!\cup\! \{ \alpha_1 \}, \Atts', \Supps', \BS' \rangle$, a GS $\SF$ satisfies \emph{neutrality} iff for any $\alpha_2 \!\in\! \Args$, where $\Atts'(\!\alpha_2\!) \!\cup\! \Supps'(\!\alpha_2\!) \!\!=\!\! \Atts(\!\alpha_2\!) \!\cup\! \Supps(\!\alpha_2\!) \!\cup\! \{ \alpha_1 \}$, $\SF^{\graph'}(\!\alpha_3\!) \!=\! \SF^\graph(\!\alpha_3\!)$ $\forall \alpha_3 \in \Atts(\!\alpha_2\!) \!\cup\! \Supps(\!\alpha_2\!)$, $\BS'(\!\alpha_2\!) \!=\! \BS(\!\alpha_2\!)$ and $\SF^{\graph'}(\!\alpha_1\!) \!=\! 0$, $\SF^{\graph'}(\!\alpha_2\!) \!=\! \SF^\graph(\!\alpha_2\!)$.
\end{property}

This property, as with stability, takes a notably different approach to the %existing properties 
T-norm semantics given their focus on CSTs, namely that the strengths of statements are calculated recursively, but should be unaffected by bottom-strength statements, which must be unsubstantiated, as illustrated below.

\begin{example}
\label{ex:6}
Let us revisit our running example and 
%introduce the statement 
let $\alpha_7 \!=\! \langle \top, d \rangle$ (``It is a fact that CO$_2$ measurements are unreliable'') with $\tau(\alpha_7) \!=\! 0$. Then, under the DC and T-norm-p semantics, $\SF^{\graph}_{\wedge_D}(\alpha_7) \!=\! \SF^{\graph}_{T_p}(\alpha_7) \!=\! 0$, i.e., $\alpha_7$ is bottom-strength. However, the strength of $\alpha_4$ under %the 
DC %semantics 
is $\SF^{\graph}_{\wedge_D}(\alpha_4) \!=\! 0.8$ (unchanged). However, under the T-norm-p semantics, it is still $\SF^{\graph}_{T_p}(\alpha_4) \!=\! 0$ (bottom-strength).
Under the DC semantics, $\alpha_7$ is bottom-strength and, thus, has no influence on the strength of $\alpha_4$. In contrast, the bottom-strength of $\alpha_7$ under T-norm-p leads to the bottom-strength of $\alpha_4$ due to the multiplicative property% of the semantics
, thus violating neutrality.
\end{example}

The next two properties consider the effect of an attacking or supporting statement being added which is not bottom-strength, taking inspiration from \emph{bi-variate monotony}~\cite{Amgoud_18}, \emph{monotonicity}~\cite{Baroni_19} and \emph{P3} in \cite{Skiba_23}. The first requires that adding an attacker should not strengthen a statement. 

\begin{property}%[Attacked Premise]
\label{prop:attackedpremise}
    %Given a second  SG
    Given two SGs $\graph$ and $\graph' = \langle \Args \cup \{ \alpha_1 \}, \Atts', \Supps', \BS' \rangle$, a GS $\SF$ satisfies \emph{attacked premise} iff for any $\alpha_2 \in \Args$, where $\Atts'(\alpha_2) \!=\! \Atts(\alpha_2) \cup \{ \alpha_1 \}$, $\Supps'(\alpha_2) \!=\! \Supps(\alpha_2)$, $\SF^{\graph'}(\alpha_3) \!=\! \SF^\graph(\alpha_3)$ $\forall \alpha_3 \!\in\! \Atts(\alpha_2) \cup \Supps(\alpha_2)$ and $\BS'(\alpha_2) \!=\! \BS(\alpha_2)$, $\SF^{\graph'}(\alpha_2) \!\leq\! \SF^\graph(\alpha_2)$.
    %\AR{Moreover, \emph{strict attacked premise} is satisfied iff $\SF^{\graph'}(\alpha_2) < \SF^\graph(\alpha_2)$.}
    %\note{AR: shall we be complete about it like this?}
    %Given a weighted SG $\graph = \langle \Args, \Atts, \Supps, \BS \rangle$, a GS $\SF$ satisfies \emph{attacked premise} iff for any $\alpha_1 \in \Args$, where $\exists x \in Prem(\alpha_1)$ such that $\exists \alpha_2 = \langle \Phi_2, \Psi_2 \rangle \in \Atts(\alpha_1)$ with $\Psi_2 \equiv x$ and $\SF^\graph(\alpha_2) > 0$, $\SF^\graph(\alpha_1) < 1$.
\end{property}

Next, adding a supporter should not weaken a statement.

\begin{property}%[Supported Premise]
\label{prop:supportedpremise}
    %Given a second  SG
    Given two SGs $\graph$ and $\graph' = \langle \Args \cup \{ \alpha_1 \}, \Atts', \Supps', \BS' \rangle$, a GS $\SF$ satisfies \emph{supported premise} iff for any $\alpha_2 \in \Args$, where $\Atts'(\alpha_2) \!=\! \Atts(\alpha_2)$, $\Supps'(\alpha_2) \!=\! \Supps(\alpha_2) \!\cup\! \{ \alpha_1 \}$, $\SF^{\graph'}(\alpha_3) \!=\! \SF^\graph(\alpha_3)$ $\forall \alpha_3 \!\in\! \Atts(\alpha_2) \cup \Supps(\alpha_2)$ and $\BS'(\alpha_2) \!=\! \BS(\alpha_2)$, $\SF^{\graph'}(\alpha_2) \!\geq\! \SF^\graph(\alpha_2)$. 
    %\AR{Moreover, \emph{strict supported premise} is satisfied iff $\SF^{\graph'}(\alpha_2) > \SF^\graph(\alpha_2)$.}
\end{property}

%\delete{Compared with attack and support  monotonicity, t}
These properties offer an iterative %\delete{more pinpointed} 
view of how the effects of added attackers and supporters should be governed, which does not depend on CSTs and thus considers cases with partially-complete or incomplete %information about statements' premises
statements. For instance, in Example~\ref{ex:6}, adding statement $\alpha_7$, which is an attacker of $\alpha_1$, leads to a decrease in $\alpha_1$'s overall strength, despite $\alpha_7$ missing a support for its premise $g$.

Next, we consider the effect of increasing the strengths of statements that attack or support other statements, as in %the existing properties of 
\emph{bi-variate reinforcement}~\cite{Amgoud_18} and \emph{monotonicity}~\cite{Baroni_19}. First, strengthening an attacker should not strengthen a statement.

\begin{property}%[Weakened Premise]
\label{prop:weakenedpremise}
    %Given a second  SG
    Given two SGs $\graph$ and $\graph' = \langle \Args', \Atts', \Supps', \BS' \rangle$, a GS $\SF$ satisfies \emph{weakened premise} iff for any $\alpha_1 \in \Args$, where $\Atts'(\alpha_1) = \Atts(\alpha_1)$, $\Supps'(\alpha_1) = \Supps(\alpha_1)$, $\BS'(\alpha_1) = \BS(\alpha_1)$, and $\SF^{\graph'}(\alpha_2) = \SF^\graph(\alpha_2)$ $\forall \alpha_2 \in \Atts(\alpha_1) \cup \Supps(\alpha_1) \setminus \{ \alpha_3 \}$, where $\alpha_3 \in \Atts(\alpha_1)$ and $\SF^{\graph'}(\alpha_3) > \SF^\graph(\alpha_3)$, $\SF^{\graph'}(\alpha_1) \leq \SF^\graph(\alpha_1)$.
\end{property}

Analogously strengthening a supporter should not weaken a statement.

\begin{property}%[Strengthened Premise]
\label{prop:strengthenedpremise}
    %Given a second  SG
    Given two SGs $\graph$ and $\graph' = \langle \Args', \Atts', \Supps', \BS' \rangle$, a GS $\SF$ satisfies \emph{strengthened premise} iff for any $\alpha_1 \in \Args$, where $\Atts'(\alpha_1) = \Atts(\alpha_1)$, $\Supps'(\alpha_1) \!=\! \Supps(\alpha_1)$, $\BS'(\alpha_1) = \BS(\alpha_1)$, and $\SF^{\graph'}(\alpha_2) = \SF^\graph(\alpha_2)$ $\forall \alpha_2 \in \Atts(\alpha_1) \cup \Supps(\alpha_1) \setminus \{ \alpha_3 \}$, where $\alpha_3 \in \Supps(\alpha_1)$ and $\SF^{\graph'}(\alpha_3) > \SF^\graph(\alpha_3)$, $\SF^{\graph'}(\alpha_1) \geq \SF^\graph(\alpha_1)$.
\end{property}

These properties %give a %\AR{n} %\delete{more} 
%immediate 
%way of 
assess added statements' effects, conditioning on their strengths rather than weights, and direct attackers and supporters rather than CSTs, %as is the case for the T-norm semantics
thus accommodating statements of all three levels of completeness.

Note that, for Properties \ref{prop:neutrality}-\ref{prop:strengthenedpremise}, the opposite effect on the statement's strength holds for the reverted modification, i.e., for removing or reducing the strengths of statements.
%Also, these properties give a way of assessing the intuitiveness of a GS in all cases iteratively from the starting point of Property \ref{prop:stability} to any given state, including the extreme cases which we consider next.

%First, a statement with a rejected premise should also be rejected.
Next, a bottom-strength premise causes a statement to be bottom-strength. This property is related to \emph{argument death} \cite{Spaans_21}.

% While this is less closely aligned with properties defined for abstract GS, there are parallels with the properties of \note{x~\cite{x}}.

\begin{property}%[Bottom-Strength Premise]
\label{prop:bottomstrengthpremise}
    A structured GS $\SF$ satisfies \emph{bottom-strength premise} iff for any $\alpha_1 \in \Args$, where 
    $\exists x \in Prem(\alpha_1)$ such that $\exists \alpha_2 = \langle \Phi_2, \Psi_2 \rangle \in \Atts(\alpha_1)$ with $\Psi_2 \equiv \neg x$ and $\SF^\graph(\alpha_2) = 1$ and $\nexists \alpha_3 =  \langle \Phi_3, \Psi_3 \rangle \in \Supps(\alpha_1)$ with $\Psi_3 \equiv x$ and $\SF^\graph(\alpha_3) > 0$, $\SF^\graph(\alpha_1) = 0$.
\end{property}

%We believe that, %in settings with 
For partially-complete or incomplete %information
statements especially, this property provides an alternative reason for requiring %that a statement is bottom-strength
rejection, i.e., there is maximal evidence to the contrary of one of its premises since they are, critically, part of a conjunction, as opposed to provability, where this is the case unless every premise in the statement is supported. Note that bottom-strength premise is also desirable for complete statements. We illustrate this property in the following.

\begin{example}
\label{ex:9}
    %Revisiting 
    In our running example, %consider the new statement 
    let $\alpha_8 = \langle e \wedge f, \neg b \rangle$ (``Extreme weather events are increasing and solar activity is high, %therefore 
    thus global temperatures are not increasing'') with $\tau(\alpha_8) \!=\! 0.5$. 
    % We now have: $\Supps = \{(\alpha_2, \alpha_1), (\alpha_3, \alpha_1), (\alpha_3, \alpha_6),  (\alpha_2, \alpha_5), (\alpha_5, \alpha_6), (\alpha_5, \alpha_8), (\alpha_7, \alpha_4) \}$ and $\Atts = \{(\alpha_4, \alpha_1), (\alpha_4, \alpha_5), (\alpha_8, \alpha_1), (\alpha_8, \alpha_6)\}$.
    Now, $\SF^{\graph}_{\wedge_D}(\alpha_8) = 0.707$ and $\SF^{\graph}_{T_p}(\alpha_8) = 0$. Let us now add a new statement $\alpha_9 = \langle \top, \neg f\rangle$ (``It is a fact that solar activity is not high'') with $\tau(\alpha_9) =1$. Then, the strength of $\alpha_8$ is $\SF^{\graph}_{\wedge_D}(\alpha_7) = \SF^{\graph^1}_{D}(e) \times \SF^{\graph^1}_{D}(g)  = 0.707 \times 0 = 0$ and $\SF^{\graph}_{T_p}(\alpha_8) = 0$.
    Both semantics now reject $\alpha_8$, but for different reasons. The DC semantics rejects it because one of its premises ($f$) is fully contradicted by $\alpha_9$, while the T-norm-p semantics already gives it bottom-strength due to incomplete support. 
    % This example shows that when there is definitive evidence against a key premise, statements relying on that premise are completely rejected, regardless of the strength of other premises.
\end{example}

%Conversely, the next property requires that a statement with all of its premises accepted should also be accepted.

Next, statements with %accepted 
universally top-strength premises should also be top-strength. 

\begin{property}%[Top-Strength Premises]
\label{prop:topstrengthpremises}
    A structured GS $\SF$ satisfies \emph{top-strength premises} iff for any $\alpha_1 \in \Args$, where $\forall x \in Prem(\alpha_1)$, $\nexists \alpha_2 = \langle \Phi_2, \Psi_2 \rangle \in \Atts(\alpha_1)$ such that $\Psi_2 \!\!\equiv\!\! \neg x$ and $\SF^\graph(\alpha_2) \!\!>\!\! 0$ and $\exists \alpha_3 \!=\! \langle \Phi_3, \Psi_3 \rangle \!\in\! \Supps(\alpha_1)$ such that $\Psi_3 \!\equiv\! x$ and $\SF^\graph(\alpha_3) \!=\! 1$, $\SF^\graph(\alpha_1) \!=\! 1$.
\end{property}

This property somewhat overrides a statement's weight when its premises are maximally supported, as in many approaches in argumentation that do not include weights. This seems intuitive particularly in our SGs, where the premises of statements (of any completeness) are conjunctions. 

\begin{example}
\label{ex:11}
%Continuing 
In our running example, %let us focus on %statements 
consider
$\alpha_5 \!=\! \langle a, e\rangle$% and 
, $\alpha_2 \!=\! \langle \top, e \rangle$, %where
with $\tau(\alpha_5) \!=\! 0.8$% and
, $\tau(\alpha_2) \!=\! 1$. Then, $\SF^\graph_{\wedge_D}(\alpha_5) \!=\! 1$, %and 
$\SF^\graph_{T_p}(\alpha_5) \!=\! 0.8$.
Despite $\alpha_2$ fully supporting the premise of $\alpha_5$ (%due to 
as the weight of $\alpha_2$ %being 
is $1$), the strength of $\alpha_5$ under the T-norm-p semantics %does not reach
is not 1. In contrast, the DC semantics assigns full strength to $\alpha_5$ when its premise is fully supported.
\end{example}

Finally, attackers and supporters may affect statements symmetrically, i.e., two statements' premises that negate one another have ``opposite'' strengths, as in \emph{duality}~\cite{Potyka_18} and \emph{franklin}~\cite{Amgoud_18}.\footnote{For the DC semantics, mirroring's conditions are triggered only for statements with single literals in their premises. However, this may not be the case for other modular structured GS.}

\begin{property}%[Mirroring]
\label{prop:mirroring}
    A structured GS $\SF$ satisfies \emph{mirroring} iff for any $\alpha_1, \alpha_2 \in \Args$, where $\alpha_1 = \langle \Phi_1, \Psi_1 \rangle$, $\alpha_2 = \langle \Phi_2, \Psi_2 \rangle$, $|Prem(\alpha_1)| = |Prem(\alpha_2)|$, $\Phi_1 \equiv \neg \Phi_2$ and $\BS(\alpha_1) = \BS(\alpha_2) = 0.5$. Then, $\SF^\graph(\alpha_1) = 1 -\SF^\graph(\alpha_2)$.
\end{property}

Mirroring seems intuitive and reflects a fairness in how statements are evaluated, potentially inspiring trust in users.

\begin{example}
\label{ex:12}
%Building on 
In our running example, %consider 
let $\alpha_{10} = \langle \neg a, \neg e \rangle$ (``CO$_2$ emissions are not rising, %therefore 
thus extreme weather events are not increasing'') with $\tau(\alpha_5) \!=\! \tau(\alpha_{10}) \!=\! 0.5$. Then, $\SF^\graph_{\wedge_D}(\alpha_{10}) \!=\! 1- \SF^\graph_{\wedge_D}(\alpha_5) \!=\! 1 - 0.55 \!=\! 0.45$, and $\SF^\graph_{T_p}(\alpha_5) \!=\! 0$ and $\SF^\graph_{T_p}(\alpha_{10}) \!=\! 0.45$.
Despite $\alpha_5$ and $\alpha_{10}$ having premises that negate one another and equal base weights, their strengths are not complementary (do not sum to 1) under T-norm-p semantics. In contrast, the DC semantics shows a desirable duality in attackers' and supporters' effects.
\end{example}

\subsection{Theoretical Analysis}
\label{ssec:theory}

\begin{table}[t]
    \centering
    \small
    \begin{tabular}{cccccccc}
        \hline
        %\multirow{2}{*}{\textbf{Property}}
        \textbf{Property} & 
        %\multicolumn{6}{c}{\textbf{Gradual Semantics}} \\ 
        %
        % & 
        \!\!$\SF_{T_p}$\!\! & 
        \!\!$\SF_{T_m}$\!\! & 
        \!\!$\SF_{\wedge_D}$\!\! & 
        \!\!$\SF_{\wedge_Q}$\!\! & 
        \!\!$\SF_{D}$\!\! & 
        \!\!$\SF_{Q}$\!\! \\
        \hline
        \ref{prop:directionality}: Directionality & 
        \textcolor{teal}{$\checkmark$} & 
        \textcolor{teal}{$\checkmark$} & 
        \textcolor{teal}{$\checkmark$} & 
        \textcolor{teal}{$\checkmark$} & 
        \textcolor{teal}{$\checkmark$} & 
        \textcolor{teal}{$\checkmark$} \\
        \ref{prop:rewriting}: Rewriting & 
        \textcolor{teal}{$\checkmark$} & 
        \textcolor{teal}{$\checkmark$} & 
        \textcolor{red}{$\times$} & 
        \textcolor{red}{$\times$} & 
        $-$ &
        $-$ \\
        \ref{prop:provability}: Provability & 
        \textcolor{teal}{$\checkmark$} & 
        \textcolor{teal}{$\checkmark$} & 
        \textcolor{red}{$\times$} & 
        \textcolor{red}{$\times$} & 
        $-$ &
        $-$ \\
        \ref{prop:weakprovability}: Weak Provability & 
        \textcolor{teal}{$\checkmark$} & 
        \textcolor{teal}{$\checkmark$} & 
        \textcolor{teal}{$\checkmark$} & 
        \textcolor{teal}{$\checkmark$} & 
        $-$ &
        $-$ \\
        \ref{prop:stability}: Stability & 
        \textcolor{red}{$\times$} & 
        \textcolor{red}{$\times$} & 
        \textcolor{teal}{$\checkmark$} & 
        \textcolor{teal}{$\checkmark$} & 
        \textcolor{teal}{$\checkmark$} & 
        \textcolor{teal}{$\checkmark$} \\
        \ref{prop:neutrality}: Neutrality & 
        \textcolor{red}{$\times$} & 
        \textcolor{red}{$\times$} & 
        \textcolor{teal}{$\checkmark$} & 
        \textcolor{teal}{$\checkmark$} & 
        \textcolor{teal}{$\checkmark$} & 
        \textcolor{teal}{$\checkmark$} \\
        %
        % \delete{\ref{prop:attackmonotonicity}: Attack Monotonicity} & 
        % \textcolor{teal}{$\checkmark$} &
        % \textcolor{teal}{$\checkmark$} &
        % \textcolor{red}{$\times$} & 
        % \textcolor{red}{$\times$} & 
        % $-$  & 
        % $-$ \\
        %
        % \delete{\ref{prop:supportmonotonicity}: Support Monotonicity} & 
        % \textcolor{teal}{$\checkmark$} &
        % \textcolor{teal}{$\checkmark$} &
        % \textcolor{red}{$\times$} & 
        % \textcolor{red}{$\times$} & 
        % $-$  & 
        % $-$ \\
        %
        \ref{prop:attackedpremise}: Attacked Premise & 
        \textcolor{teal}{$\checkmark$} & 
        \textcolor{teal}{$\checkmark$} & 
        \textcolor{teal}{$\checkmark$} & 
        \textcolor{teal}{$\checkmark$} & 
        \textcolor{teal}{$\checkmark$} & 
        \textcolor{teal}{$\checkmark$} \\
        \ref{prop:supportedpremise}: Supported Premise & 
        \textcolor{teal}{$\checkmark$} & 
        \textcolor{teal}{$\checkmark$} & 
        \textcolor{teal}{$\checkmark$} & 
        \textcolor{teal}{$\checkmark$} & 
        \textcolor{teal}{$\checkmark$} & 
        \textcolor{teal}{$\checkmark$} \\
         \ref{prop:weakenedpremise}: Weakened Premise & 
        \textcolor{red}{$\times$} & 
        \textcolor{red}{$\times$} & 
        \textcolor{teal}{$\checkmark$} & 
        \textcolor{teal}{$\checkmark$} & 
        \textcolor{teal}{$\checkmark$} & 
        \textcolor{teal}{$\checkmark$} \\
        \!\!\ref{prop:strengthenedpremise}: Strengthened Premise\!\! & 
        \textcolor{teal}{$\checkmark$} & 
        \textcolor{teal}{$\checkmark$} & 
        \textcolor{teal}{$\checkmark$} & 
        \textcolor{teal}{$\checkmark$} & 
        \textcolor{teal}{$\checkmark$} & 
        \textcolor{teal}{$\checkmark$} \\
        \!\!\ref{prop:bottomstrengthpremise}: Bottom-Strength Premise\!\! & 
        \textcolor{teal}{$\checkmark$} & 
        \textcolor{teal}{$\checkmark$} &
        \textcolor{teal}{$\checkmark$} & 
        \textcolor{red}{$\times$} & 
        $-$ &
        $-$ \\
        \!\!\ref{prop:topstrengthpremises}: Top-Strength Premises\!\! & 
        \textcolor{red}{$\times$} & 
        \textcolor{red}{$\times$} & 
        \textcolor{teal}{$\checkmark$} & 
        \textcolor{red}{$\times$} & 
        $-$ &
        $-$ \\
        \ref{prop:mirroring}: Mirroring & 
        \textcolor{red}{$\times$} & 
        \textcolor{red}{$\times$} & 
        \textcolor{teal}{$\checkmark$} & 
        \textcolor{teal}{$\checkmark$} & 
        $-$ &
        $-$ \\
        \hline
    \end{tabular}
    \caption{Structured ($\SF_{T_p}$, $\SF_{T_m}$, $\SF_{\wedge_D}$ and $\SF_{\wedge_Q}$) and abstract ($\SF_{D}$ and $\SF_{Q}$) GS' satisfaction (denoted with \textcolor{teal}{$\checkmark$}) or violation (denoted with \textcolor{red}{$\times$}) of the assessed %existing (\ref{prop:provability}-\ref{prop:supportmonotonicity}) and novel (\ref{prop:stability}-\ref{prop:mirroring}) 
    properties, where %an  abstract GS' 
    incompatibility of a GS with a property
    is denoted with $-$. %\note{AR: use darker green?}
    }
    \label{tab:properties}
\end{table}

We now %perform theoretical analyses assessing 
assess whether instantiations of the DC semantics, as well as existing structured and abstract GS, satisfy the properties from §\ref{ssec:properties}. Table \ref{tab:properties} summarises the results.%\footnote{\note{delete if we keep propositions} Full propositions and proofs are given in Appendix \ref{appendix:proofs}.}

First, we assess the \textbf{T-norm semantics} %using the two De Morgan triples that we defined 
in §\ref{sec:prelim},\! 
$\SF_{T_p}$ and $\SF_{T_m}$\!.
%namely T-norm-p semantics ($\SF_{T_p}$) and T-norm-m semantics ($\SF_{T_m}$).
\begin{proposition}
    $\SF_{T_p}$ and $\SF_{T_m}$ satisfy Properties \ref{prop:directionality}-\ref{prop:weakprovability}, \ref{prop:attackedpremise}-\ref{prop:supportedpremise}, and  \ref{prop:strengthenedpremise}-\ref{prop:bottomstrengthpremise}, but violate Properties \ref{prop:stability}-\ref{prop:neutrality}, \ref{prop:weakenedpremise}, and \ref{prop:topstrengthpremises}-\ref{prop:mirroring}.
\end{proposition}
Both GS satisfy identical properties across our collection, %, including those which were taken from the literature%, as might be expected given that their design is centred on CSTs
performing poorly when we consider the novel properties.
Even stability and neutrality, two basic properties in the realm of abstract GS, are violated due to the fact that they do not consider the weight as a starting point, nor the strength as a recursive function, respectively.
Attacked and supported premises are both satisfied, while only weakened, and not strengthened, premise is satisfied due to the asymmetry in the way the attackers and supporters are evaluated, which leads to mirroring being violated also. 
Similarly, bottom-strength premises is satisfied, but top-strength premises is violated due to the (potentially non-maximal) weight of the statement playing a role in the calculation.
Given top-strength premises' fundamental role across statements of any completeness, this represents a significant weakness.

Next, we consider the \textbf{DC semantics}, supported by either the DF-QuAD semantics ($\SF_{\wedge_D}$) or the QEM semantics ($\SF_{\wedge_Q}$), as defined in §\ref{ssec:methodology}.
\begin{proposition}
    $\SF_{\wedge_D}$ satisfies Properties \ref{prop:directionality} and \ref{prop:weakprovability}-\ref{prop:mirroring}, but violates Properties \ref{prop:rewriting}-\ref{prop:provability}. %and \ref{prop:attackreinforcement}-\ref{prop:supportmonotonicity}.
     $\SF_{\wedge_Q}$ satisfies Properties \ref{prop:directionality}, \ref{prop:weakprovability}-\ref{prop:strengthenedpremise}, and \ref{prop:mirroring}, but violates Properties \ref{prop:rewriting}-\ref{prop:provability} and \ref{prop:bottomstrengthpremise}-\ref{prop:topstrengthpremises}.
\end{proposition}
Both DC semantics only satisfy directionality of the existing properties. 
The violation of the other existing properties (rewriting and provability) can be justified by their incompatibility with the novel properties which we introduce.

\begin{proposition}
  Rewriting, stability and top-strength premises are incompatible.
\end{proposition}

\begin{proposition}
   Provability and stability are incompatible.
\end{proposition}

This means that a user selecting a semantics must choose between its satisfation of one set of properties or another, so our approach represents a viable alternative to the existing work, particularly in cases with incomplete information.
Examples \ref{ex:3} and \ref{ex:4} demonstrate what we believe is further justification for this. 
For example, rewriting is not satisfied by either semantics since more evidence, even if the statements constituting it are not complete, will strengthen a statement, which seems to be reasonable particularly for our settings with partially-complete or incomplete statements.
Moreover, provability is violated by design, since we wanted GS that do not necessarily reject a statement when its support is missing, only when the statement's weight is also minimal, hence the GS' satisfaction of weak provability instead. 
Both GS align with stability in treating the weight as a ``starting point'', and from there they satisfy attacked and supported premises, as well as weakened and strengthened premises, due to the recursive way they handle attackers and supporters.
Interestingly, the DC semantics, when supported by the DF-QuAD semantics, satisfies bottom-strength premise and top-strength premises, %where that
but it does not when supported by the QEM semantics% does not
. This is due to the saturation effect that is seen at the extremities of the strength scale in the DF-QuAD semantics, giving rise to this intuitive behaviour for this specific type of SG (with conjunctions as premises).
Finally, mirroring is satisfied by both GS.

We now assess the two abstract GS, the \textbf{DF-QuAD semantics} ($\SF_{D}$) and the \textbf{QEM semantics} ($\SF_{Q}$), as defined in §\ref{sec:prelim}, along the compatible properties.
\begin{proposition}
    $\SF_{D}$ and $\SF_{Q}$ satisfy Properties \ref{prop:directionality} and \ref{prop:stability}-\ref{prop:strengthenedpremise}.
    %, but violates \note{x,y,z}.
    (They are incompatible with the other properties.)
\end{proposition}
Although both of the abstract GS satisfy all of the properties with which they are compatible, they are incompatible with many others, showing the advantages of structured GS in considering statements' logical structure.

While both T-norm semantics perform well in the existing properties which require complete information, under incomplete information, i.e., with partially-complete or incomplete statements, where the novel properties are particularly %useful
powerful, there are clear weaknesses.
This highlights how the DC semantics, and in particular that  supported by DF-QuAD, offer clear advantages, especially given that all of the properties cannot be satisfied by a given semantics.
Thus, in settings with varying completeness of information, our approach fills a gap in the literature.

Finally, we give a theorem that demonstrates the conditions under which a DC semantics, supported by a specific abstract GS, aligns with the abstract GS itself, namely when the premises of all statements consist of at most one literal (meaning SGs effectively become QBAFs).

\begin{theorem}
    Let $\Args$ be such that $\forall \alpha \in \Args$, $|Prem(\alpha)| \leq 1$. Then, for any given abstract GS $\SF_i$, $\SF_{\wedge_i} = \SF_i$. 
\end{theorem}

\section{Conclusions and
Future Work}
\label{sec:conc}

In this paper, we %have 
introduced a novel methodology for obtaining GS in SGs that naturally accommodate incomplete information. %Particularly, we presented a 
Our modular methodology %for structured GS that 
dialectically evaluates the literals in a statement's premise before aggregating these evaluations based on the statement's construction. This separation allows our semantics to handle incomplete information by leveraging any existing GS for QBAFs to evaluate the  available evidence without requiring complete support. We then demonstrated how our %dialectical conjunction 
DC semantics %demonstrated how 
can effectively leverage
abstract %gradual semantics 
GS %can be effectively leveraged 
for structured argumentation. Furthermore, we discussed existing and novel properties for GS in SGs, revealing some %tensions 
incompatibilities between existing properties, e.g., %like 
rewriting and provability, and our novel properties, e.g., %such as 
stability and top-strength premises, showing that users must select between them given a particular contextual setting, e.g., based on the completeness notions we define. Our theoretical analysis demonstrated that the DC semantics supported by the DF-QuAD abstract GS satisfies all of our novel properties, while %T-norm semantics 
existing structured GS struggle with cases without complete information.%incomplete %\delete{cases}

While this paper establishes the foundations of our novel approach using a toy example for pedagogical reasons, we see significant potential for real-world applications, e.g., in the analysis of debates mapped onto SGs through argument mining \cite{lawrence-reed-2019-argument}, possibly enhanced by large language models \cite{Cabessa_25,Gorur_25}. It is well known that mining structured arguments is a complex problem, not least due to the presence of enthymemes causing incompleteness \cite{hunter2022understanding,stahl-etal-2023-mind}. A methodology leveraging GS that can handle incomplete information could alleviate some of the burden that structured argumentation's requirements for completeness place on these methods.

Our methodology represents an early step forward in GS for structured argumentation, and %it 
opens 
%the novelty of our methodology, it highlights 
various potentially fruitful directions for future work, in particular targeting more general settings of structured argumentation, e.g., in ABA or ASPIC$^+$.
First, we would like to broaden our analysis, both within our restricted form of SGs, exploring various other instantiations of our DC semantics, e.g., using other abstract GS such as the \emph{exponent-based restricted semantics}~\cite{Amgoud_18}, as well as extending to different forms of SGs with more complex logical premises.
%For example, 
Specifically, our ability to deal with partially-complete or incomplete statements may be related to the possibility of deriving assumptions in non-flat ABA~\cite{Cyras_17}.
%i.e., the \emph{horn ABA} instance \cite{x}% or equivalents in ASPIC$^+$ \cite{blah}
We would also like to consider whether other properties for abstract GS are desirable%in these settings
, e.g., \emph{open-mindedness}~\cite{Potyka_19} (the satisfaction of which gives a clear advantage of $\SF_Q$ over $\SF_D$), \emph{attainability}~\cite{Cocarascu_19}, whose suitability in our %restricted 
SGs was not obvious due to the conjunctive premise.
It would also be interesting to formalise the interplay between properties, % for abstract and structured GS, e.g., if an abstract GS satisfying one property guarantees a structured instantiation will satisfy another property, or compatibility between properties, 
e.g., between provability and stability, which are incompatible without restrictions.
Further,
it would be interesting to %see 
investigate whether our analysis could be applicable to probabilistic argumentation~\cite{Kohlas_03,Dung_10,%Thimm_12,Hunter_13,
Hunter_14,Gabbay_15_prob,%Hunter_17,Hunter_20,
Fazzinga_18}. %where uncertainty over arguments is modelled. 
%However,
%but in all of these works, either support relations are not considered or strengths are associated with probabilities and/or are not assigned to individual arguments. 
Finally, we %would like 
plan to explore how our methodology %and the associated properties 
translates to real-world applications, %with a study such as~\cite{Rago_18_SUM}. 
%The 
e.g., in multi-agent model reconciliation \cite{vasileiou2024dialectical} %seems to have much potential, not least 
given GS' capability for modelling human or machine reasoning~\cite{Rago_23}.
%Further, it seems that handling and leveraging similarity between arguments~\cite{Amgoud_18_Accounting,Budan_20,Amgoud_21}, which could be supported by our logical formalism and driven by the measure used in~\cite{Amgoud_18_Similarity}, could be particularly useful under incomplete information.

\section*{Acknowledgements}

Antonio Rago and Francesca Toni are partially funded by the European Research Council (ERC) under the European
Union’s Horizon 2020 research and innovation programme (grant agreement No. 101020934)
{and by J.P. Morgan and the Royal Academy
of Engineering under the Research Chairs and Senior Research
Fellowships scheme (grant agreement no. RCSRF2021\textbackslash 11\textbackslash 45)}.
Stylianos Loukas Vasileiou and William Yeoh are partially supported by the National Science Foundation (NSF) under award 2232055. 
Son Tran is partially supported by the NSF grants \#1914635, \#2151254, and ExpandAI grant \#2025-67022-44266. 
The views and conclusions contained in this document are those of the authors and should not be interpreted as representing the official policies, either expressed or implied, of the sponsoring organizations, agencies, or governments.

\bibliographystyle{kr}
\bibliography{bib}

\newpage
\quad
\newpage

\appendix

\section{Additional Properties}
\label{appendix:props}

In this section, we assess the GS against the properties for CSTs from \cite{Jedwabny_20}

The first two properties concern how increasing the weight of a statement affects another statement wrt CSTs. First, increasing the weight of a statement in a CST attacking the CST of another statement should not strengthen the latter statement.

\begin{property}[Attack Reinforcement]
\label{prop:attackreinforcement}
    Given a second SG $\graph' = \langle \Args, \Atts, \Supps, \BS' \rangle$, a structured GS $\SF$ satisfies \emph{attack reinforcement} iff for any $\alpha_1, \alpha_2 \in \Args$ where:
    \begin{itemize}
        \item $\nexists T \in \CST^\graph(\alpha_2)$ such that $\alpha_1 \in T$;
        \item $\exists \alpha_3 \in \Args \setminus \{ \alpha_1, \alpha_2 \}$ such that $\exists T' \in \CST^\graph(\alpha_3)$ where $T'$ attacks some $T \in \CST^{\graph'}(\alpha_2)$ and $\alpha_1 \in T'$;
        \item $\BS'(\alpha_1) \geq \BS(\alpha_1)$; and
        \item $\BS'(\alpha_4) = \BS(\alpha_4)$ $\forall \alpha_4 \in \Args \setminus \{ \alpha_1 \}$;
    \end{itemize}
    it holds that $\SF^{\graph'}(\alpha_2) \leq \SF^\graph(\alpha_2)$.
\end{property}

Then, increasing the weight of a statement in a statement's CST should not weaken the latter statement.

\begin{property}[Support Reinforcement]
\label{prop:supportreinforcement}
    Given a second SG $\graph' = \langle \Args, \Atts, \Supps, \BS' \rangle$, a structured GS $\SF$ satisfies \emph{support reinforcement} iff for any $\alpha_1, \alpha_2 \in \Args$ where:
    \begin{itemize}
        \item $\exists T \in \CST^\graph(\alpha_2)$ such that $\alpha_1 \in T$;
        \item $\nexists \alpha_3 \in \Args \setminus \{ \alpha_1, \alpha_2 \}$ such that $\exists T' \in \CST^\graph(\alpha_3)$ where $T'$ attacks some $T \in \CST^{\graph'}(\alpha_2)$ and $\alpha_1 \in T'$;
        \item $\BS'(\alpha_1) \geq \BS(\alpha_1)$; and
        \item $\BS'(\alpha_4) = \BS(\alpha_4)$ $\forall \alpha_4 \in \Args \setminus \{ \alpha_1 \}$;
    \end{itemize}
    it holds that $\SF^{\graph'}(\alpha_2) \geq \SF^\graph(\alpha_2)$.
\end{property}

These properties, while intuitive, are focused around CSTs and do not govern the behaviour of GS under incomplete information, e.g., when part of a statement's support is not known.

The final two existing properties consider the effects of adding statements to CSTs. First, adding a statement to a CST attacking another statement's CST should not strengthen the latter statement.

\begin{property}[Attack Monotonicity]
\label{prop:attackmonotonicity}
     Given a second SG $\graph' = \langle \Args \cup \{ \alpha_1 \}, \Atts', \Supps', \BS' \rangle$, where $\alpha_2 = \alpha_1$ or $\alpha_3 = \alpha_1$ $\forall (\alpha_2,\alpha_3) \in (\Atts' \cup \Supps') \setminus (\Atts \cup \Supps)$, and $\BS'(\alpha_4) = \BS(\alpha_4)$ $\forall \alpha_4 \in \Args$, a structured GS $\SF$ satisfies \emph{attack monotonicity} iff for any $\alpha_5 \in \Args$ where:
    \begin{itemize}
        \item $\nexists T \in \CST^{\graph'}(\alpha_5)$ such that $\alpha_1 \in T$;
        \item $\exists \alpha_6 \in \Args \setminus \{ \alpha_1, \alpha_5 \}$ such that $\exists T' \in \CST^\graph(\alpha_6)$ where $T'$ attacks some $T \in \CST^{\graph'}(\alpha_5)$ and $\alpha_1 \in T'$;
        %
        %\item \note{$\BS'(\alpha_1) \geq \BS(\alpha_1)$; and}
        %
        %\item \note{$\BS'(\alpha_4) = \BS(\alpha_4)$ $\forall \alpha_4 \in \Args \setminus \{ \alpha_1 \}$;}
        %
    \end{itemize}
    it holds that $\SF^{\graph'}(\alpha_5) \leq \SF^\graph(\alpha_5)$.
\end{property}

Then, adding a statement to another statement's CST should not weaken the latter statement.

\begin{property}[Support Monotonicity]
\label{prop:supportmonotonicity}
     Given a second SG $\graph' = \langle \Args \cup \{ \alpha_1 \}, \Atts', \Supps', \BS' \rangle$, where $\alpha_2 = \alpha_1$ or $\alpha_3 = \alpha_1$ $\forall (\alpha_2,\alpha_3) \in (\Atts' \cup \Supps') \setminus (\Atts \cup \Supps)$, and $\BS'(\alpha_4) = \BS(\alpha_4)$ $\forall \alpha_4 \in \Args$, a structured GS $\SF$ satisfies \emph{support monotonicity} iff for any $\alpha_5 \in \Args$ where:
    \begin{itemize}
        \item $\exists T \in \CST^{\graph'}(\alpha_5)$ such that $\alpha_1 \in T$;
        \item $\nexists \alpha_6 \in \Args \setminus \{ \alpha_1, \alpha_5 \}$ such that $\exists T' \in \CST^\graph(\alpha_6)$ where $T'$ attacks some $T \in \CST^{\graph'}(\alpha_5)$ and $\alpha_1 \in T'$;
        %
        %\item \note{$\BS'(\alpha_1) \geq \BS(\alpha_1)$; and}
        %
        %\item \note{$\BS'(\alpha_4) = \BS(\alpha_4)$ $\forall \alpha_4 \in \Args \setminus \{ \alpha_1 \}$;}
        %
    \end{itemize}
    it holds that $\SF^{\graph'}(\alpha_5) \geq \SF^\graph(\alpha_5)$.
\end{property}

Once again, these properties appear desirable when operating with perfect information, but their focus on CSTs ignores the effect of attackers or supporters under incomplete information.

Appendix \ref{appendix:proofs} contains theoretical results assessing the four structured GS against these four properties (they are incompatible with abstract GS).
In summary, $\SF_{T_p}$ and $\SF_{T_m}$ satisfy Properties \ref{prop:attackreinforcement}-\ref{prop:supportmonotonicity}, while $\SF_{\wedge_D}$ and $\SF_{\wedge_Q}$ violate them.

While each of these four properties differentiates our introduced GS from the existing GS, we believe this is reasonable given the properties' focus on CSTs.
Further, it can be seen that Properties \ref{prop:attackedpremise} to \ref{prop:strengthenedpremise} represent alternatives which guarantee suitable behaviour under all states of completeness (Definition \ref{def:completeness}).

\newpage

\section{Proofs}
\label{appendix:proofs}

% NOTE, if we publish these proofs, the notation needs fixing

Here we give the proofs for the theoretical work in the paper.

\begin{dummycorollary}
    \!Given an SG $%\graph \!=\!
    \langle \Args,\! \Atts,\! \Supps,\! \BS \rangle$,\! $\Atts \cap \Supps \!=\! \emptyset$.
\end{dummycorollary}

\begin{proof}
    Let us prove by contradiction. If $\exists (\langle \Phi_1, \Psi_1 \rangle,\langle \Phi_2, \Psi_2 \rangle) \in \Atts \cap \Supps$, then, by Definition \ref{def:supp-att}, $\Psi_1, \neg \Psi_1 \in Prem(\langle \Phi_2, \Psi_2 \rangle)$, which cannot be the case since $\Phi_2$ must be a consistent compound, and so we have the contradiction.
\end{proof}

\begin{dummycorollary}
    For any abstract GS $\SF_*$ that satisfies existence and uniqueness, %i.e., for any $\alpha \in \Args$, $\SF^\graph_{*}(\alpha)$ always exists and is unique, 
    $\SF_{\wedge_*}$ also satisfies existence and uniqueness.
\end{dummycorollary}

\begin{proof}
    This follows from Definition \ref{def:DC} given that for any $\alpha \in \Args$, $\SF^\graph_{\wedge_*}(\alpha)$ is either equal to $\BS(\alpha)$ or a product of $\SF_{\wedge_*}^\graph(x)$ for $x \in Prem(\alpha)$.
\end{proof}

\begin{dummycorollary}
\label{cor:weakprovability}
    A structured GS $\SF$ which satisfies provability necessarily satisfies weak provability.
\end{dummycorollary}

\begin{proof}
    This follows from Properties \ref{prop:provability} and \ref{prop:weakprovability} given that the former is a more general case than the latter.
\end{proof}

\begin{dummyproposition}
    $\SF_{T_p}$ and $\SF_{T_m}$ satisfy Properties \ref{prop:directionality}-\ref{prop:weakprovability}, \ref{prop:attackedpremise}-\ref{prop:supportedpremise}, and  \ref{prop:strengthenedpremise}-\ref{prop:bottomstrengthpremise}, but violate Properties \ref{prop:stability}-\ref{prop:neutrality}, \ref{prop:weakenedpremise}, and \ref{prop:topstrengthpremises}-\ref{prop:mirroring}.
\end{dummyproposition}

\begin{proof}
    
    $\SF_{T_p}$:

    The proofs for Properties \ref{prop:directionality}-\ref{prop:provability} can be found in~\cite{Jedwabny_20}.

    \squishlist

    \item[] Property \ref{prop:weakprovability}: Weak Provability.
    By the fact that $\SF^{\graph}_{T_p}$ satisfies Provability and Corollary \ref{cor:weakprovability}, it must also satisfy Weak Provability.
    
    \item[]  Property \ref{prop:stability}: Stability (Counterexample). Let $\Args = \{ \alpha \}$, where $\alpha = \langle b, a \rangle$, $\BS(\alpha) = 1$ and $\Atts(\alpha) = \Supps (\alpha) = \emptyset$. Then, by Definition~\ref{def:Tnorm}, we have that 
    %\AR{$\bigoplus_{T \in \CST^\graph(\alpha)} \mathcal{O}(T) = 0$ since $\CST^\graph(\alpha) = \emptyset$, and thus} 
    $\SF^\graph_{T_p}(\alpha) = 0$.

    \item[] Property \ref{prop:neutrality}: Neutrality (Counterexample). Let $\Args = \{ \alpha_1, \alpha_2, \alpha_3, \alpha_4, \alpha_5 \}$ and $\Args' = \Args \cup \{ \alpha_6 \}$, where 
    $\alpha_1 = \langle a , b \rangle$, 
    $\alpha_2 = \langle \top, a \rangle$, 
    $\alpha_3 = \langle \top, c \rangle$,
    $\alpha_4 = \langle d, \neg c \rangle$,
    $\alpha_5 = \langle \top, d \rangle$,
    $\alpha_6 = \langle c, \neg a \rangle$.
    By Definition \ref{def:supp-att}, we have that $\Atts' = \Atts \cup \{ (\alpha_6, \alpha_1), (\alpha_4, \alpha_6)\}$, $\Supps' = \Supps \cup \{(\alpha_3, \alpha_6) \}$. Then let $\BS(\alpha_1) = \BS'(\alpha_1) = \BS(\alpha_2) = \BS'(\alpha_2) = \BS(\alpha_3) = \BS'(\alpha_3) = \BS(\alpha_4) = \BS'(\alpha_4) = \BS(\alpha_5) = \BS'(\alpha_5) = 1$, and $\BS'(\alpha_6) = 0.5$. By Definition \ref{def:Tnorm}, this gives $\SF^{\graph}_{T_p}(\alpha_1) = 1$ and $\SF^{\graph'}_{T_p}(\alpha_6) = 0$. Meanwhile, by the same definition, $\SF^{\graph'}_{T_p}(\alpha_1) = 0.5$. Thus, $\SF^{\graph'}_{T_p}(\alpha_1) \neq \SF^{\graph}_{T_p}(\alpha_1)$.

    \item[] Property \ref{prop:attackedpremise}: Attacked Premise. It can be seen from Definition \ref{def:Tnorm} that there are two ways in which the addition of $\alpha_1$ may affect $\SF^{\graph'}_{T_p}(\alpha_2)$: (i) $\exists T_+ \in \CST^\graph(\alpha_2)$ where $\alpha_1 \in T_+$; and (ii) $\exists \alpha_3 \in \Args' \setminus \{ \alpha_2 \}$ such that $\exists T_- \in \CST^\graph(\alpha_3)$ where $T_-$ attacks some $T \in \CST^{\graph'}(\alpha_2)$ and $\alpha_1 \in T_-$.
    By Definition \ref{def:CST}, we can see that case (i) cannot apply since $\alpha_1 \in \Atts'(\alpha_2)$.
    Then, for case (ii), it can be seen from Definition \ref{def:Tnorm} that if $\tau'(\alpha_1) = 0$, then $\SF^{\graph'}_{T_p}(\alpha_2) = \SF^{\graph}_{T_p}(\alpha_2)$, while if $\tau'(\alpha_1) >0$, then $\SF^{\graph'}_{T_p}(\alpha_2) < \SF^{\graph}_{T_p}(\alpha_2)$. Therefore, $\alpha_1$ can only reduce $\SF^{\graph'}_{T_p}(\alpha_2)$, i.e., $\SF^{\graph'}_{T_p}(\alpha_2) \leq \SF^{\graph}_{T_p}(\alpha_2)$.

    \item[] Property \ref{prop:supportedpremise}: Supported Premise. It can be seen from Definition \ref{def:Tnorm} that there are two ways in which the addition of $\alpha_1$ may affect $\SF^{\graph'}_{T_p}(\alpha_2)$: (i) $\exists T_+ \in \CST^\graph(\alpha_2)$ where $\alpha_1 \in T_+$; and (ii) $\exists \alpha_3 \in \Args' \setminus \{ \alpha_2 \}$ such that $\exists T_- \in \CST^\graph(\alpha_3)$ where $T_-$ attacks some $T \in \CST^{\graph'}(\alpha_2)$ and $\alpha_1 \in T_-$.
    If case (ii) applies, then it can only be the case that $\mathcal{O}'(T_-) = \mathcal{O}(T_-)$ due to the condition that the strength of all prior supporters and attackers of $\alpha_2$ must be equal, and so in this case, $\alpha_1$ will have no effect on $\SF^{\graph'}_{T_p}(\alpha_2)$. For case (i), it can be seen from Definition \ref{def:Tnorm} that $\alpha_1$, regardless of the value of $\SF^{\graph'}_{T_p}(\alpha_1)$, can only increase $\SF^{\graph'}_{T_p}(\alpha_2)$, as $\bigoplus$ is monotonically increasing. Therefore, we have that $\SF^{\graph'}_{T_p}(\alpha_2) \geq \SF^{\graph}_{T_p}(\alpha_2)$.

    \item[] Property \ref{prop:weakenedpremise}: Weakened Premise (Counterexample). Let 
    $\Args = \{ \alpha_1, \alpha_2,\alpha_3, \alpha_4, \alpha_5, \alpha_6, \alpha_7, \alpha_8 \}$, where 
    $\alpha_1 = \langle a\wedge b, c \rangle$, 
    $\alpha_2 = \langle \top, a \rangle$, 
    $\alpha_3 = \langle \top, b \rangle $, 
    $\alpha_4 = \langle d, e \rangle $, and 
    $\alpha_5 = \langle e, \neg b \rangle$, 
    $\alpha_6 = \langle \top, d \rangle$, 
    $\alpha_7 = \langle f, \neg d \rangle$, and 
    $\alpha_8 = \langle \top, f \rangle$. 
    Then, let $\Args' = \Args$ and 
    $\tau(\alpha_1) = \tau'(\alpha_1) =
    \tau(\alpha_2) = \tau'(\alpha_2) =
    \tau(\alpha_3) = \tau'(\alpha_3) =
    \tau(\alpha_4) = 
    \tau(\alpha_5) = \tau'(\alpha_5) =
    \tau(\alpha_6) = \tau'(\alpha_6) = 0.5$, 
    $\tau'(\alpha_4) = 0.4$
    $\tau(\alpha_7) = 0.9$,
    $\tau'(\alpha_7) = 0.5$ and 
    $\tau(\alpha_8) = \tau'(\alpha_8) = 1$. 
    Then, by Definition~\ref{def:Tnorm}, we have that $\SF^\graph_{T_p}(\alpha_1) = 0.1093$ and $\SF^\graph_{T_p}(\alpha_5) = 0.0125$. However, from the same definition, we get $\SF^{\graph'}_{T_p}(\alpha_1) = 0.1125$ and $\SF^{\graph'}_{T_p}(\alpha_5) = 0.05$, thus $\SF^{\graph'}_{T_p}(\alpha_1) > \SF^\graph_{T_p}(\alpha_1)$.

    \item[] Property \ref{prop:strengthenedpremise}: Strengthened Premise.
    It follows directly from Definition~\ref{def:Tnorm} that increasing the strength of $\alpha_3 \in \Supps(\alpha_1)$, i.e., $\SF^{\graph'}_{T_p}(\alpha_3) > \SF^{\graph}_{T_p}(\alpha_3)$ while $\forall \alpha_2 \in \Atts(\alpha_1) \cup \Supps(\alpha_1)$, $\SF^{\graph'}_{T_p}(\alpha_2) = \SF^{\graph}_{T_p}(\alpha_2)$, then the strength of $\alpha_1$ can only increase, and so we have that $\SF^{\graph'}_{T_p}(\alpha_1) \geq \SF^{\graph}_{T_p}(\alpha_1)$.

    \item[] Property \ref{prop:bottomstrengthpremise}: Bottom-Strength Premise. 
    Let us prove by contradiction. 
    If we let $\SF^\graph_{T_p}(\alpha_1) > 0$, we can see from Definition \ref{def:Tnorm} that $\exists T_+ \in \CST^\graph(\alpha_1)$ where $\mathcal{O}(T_+) > 0$. However, since $\alpha_2 \in \Atts(\alpha_1)$ and $\SF^\graph_{T_p}(\alpha_2) = 1$, we can see from Definition \ref{def:Tnorm} that $\exists T_- \in \CST^\graph(\alpha_2)$ where $\mathcal{I}(T_-) = 1$ and exists $\alpha_3 \in T_-$ such that $\BS(\alpha_3) = 1$. 
    Thus, any $T_+$ must necessarily be attacked by $T_-$ and thus $\mathcal{O}(T_+) = 0$ and we have the contradiction.

    \item[] Property \ref{prop:topstrengthpremises}: Top-Strength Premises (Counterexample). Let $\Args = \{ \alpha_1, \alpha_2 \}$, where $\alpha_1 = \langle a, b \rangle $ and $\alpha_2 = \langle \top, a \rangle $, and $\tau(\alpha_1) = 0.8$ and $\tau(\alpha_2)=1$. Then, by Definition~\ref{def:Tnorm}, we have that $\SF^\graph_{T_p}(\alpha_2) = 1$, but $\SF^\graph_{T_p}(\alpha_1) = 0.8$.

    \item[] Property \ref{prop:mirroring}: Mirroring (Counterexample). Let $\Args = \{ \alpha_1, \alpha_2,\alpha_3, \alpha_4, \alpha_5 \}$, where $\alpha_1 = \langle a, b \rangle$, $\alpha_2 = \langle \top, a \rangle$, $\alpha_3 = \langle \neg a, c \rangle $, $\alpha_4 = \langle d, \neg a \rangle $, and $\alpha_5 = \langle \top, d \rangle $. Also, let $\tau(\alpha_1) =\tau(\alpha_2)= \tau(\alpha_3) = \tau(\alpha_4) = \tau(\alpha_5) = 0.5$. Then, by Definition~\ref{def:Tnorm}, we have that $\SF^\graph_{T_p}(\alpha_1) = 0.187$ and $\SF^\graph_{T_p}(\alpha_3) = 0.0625$.

    \squishend
    
    $\SF_{T_m}$:

    The proofs for Properties \ref{prop:directionality}-\ref{prop:provability} are analogous to those for $\SF_{T_p}$ in~\cite{Jedwabny_20}.

    \squishlist

    \item[] Property \ref{prop:weakprovability}: Weak Provability.
    By the fact that $\SF^{\graph}_{T_m}$ satisfies Provability and Corollary \ref{cor:weakprovability}, it must also satisfy Weak Provability.
    
    \item[] Property \ref{prop:stability}: Stability (Counterexample). Let $\Args = \{ \alpha \}$, where $\alpha = \langle b, a \rangle$, $\BS(\alpha) = 1$ and $\Atts(\alpha) = \Supps (\alpha) = \emptyset$. Then, by Definition~\ref{def:Tnorm}, we have that 
    $\SF^\graph_{T_m}(\alpha) = 0$.

    \item[] Property \ref{prop:neutrality}: Neutrality (Counterexample). Let $\Args = \{ \alpha_1, \alpha_2, \alpha_3, \alpha_4, \alpha_5 \}$ and $\Args' = \Args \cup \{ \alpha_6 \}$, where 
    $\alpha_1 = \langle a , b \rangle$, 
    $\alpha_2 = \langle \top, a \rangle$, 
    $\alpha_3 = \langle \top, c \rangle$,
    $\alpha_4 = \langle d, \neg c \rangle$,
    $\alpha_5 = \langle \top, d \rangle$,
    $\alpha_6 = \langle c, \neg a \rangle$.
    By Definition \ref{def:supp-att}, we have that $\Atts' = \Atts \cup \{ (\alpha_6, \alpha_1), (\alpha_4, \alpha_6)\}$, $\Supps' = \Supps \cup \{(\alpha_3, \alpha_6) \}$. Then let $\BS(\alpha_1) = \BS'(\alpha_1) = \BS(\alpha_2) = \BS'(\alpha_2) = \BS(\alpha_3) = \BS'(\alpha_3) = \BS(\alpha_4) = \BS'(\alpha_4) = \BS(\alpha_5) = \BS'(\alpha_5) = 1$, and $\BS'(\alpha_6) = 0.5$. By Definition \ref{def:Tnorm}, this gives $\SF^{\graph}_{T_m}(\alpha_1) = 1$ and $\SF^{\graph'}_{T_m}(\alpha_6) = 0$. Meanwhile, by the same definition, $\SF^{\graph'}_{T_m}(\alpha_1) = 0$. Thus, $\SF^{\graph'}_{T_m}(\alpha_1) \neq \SF^{\graph}_{T_m}(\alpha_1)$.

    \item[] Property \ref{prop:attackedpremise}: Attacked Premise. It can be seen from Definition \ref{def:Tnorm} that there are two ways in which the addition of $\alpha_1$ may affect $\SF^{\graph'}_{T_m}(\alpha_2)$: (i) $\exists T_+ \in \CST^\graph(\alpha_2)$ where $\alpha_1 \in T_+$; and (ii) $\exists \alpha_3 \in \Args' \setminus \{ \alpha_2 \}$ such that $\exists T_- \in \CST^\graph(\alpha_3)$ where $T_-$ attacks some $T \in \CST^{\graph'}(\alpha_2)$ and $\alpha_1 \in T_-$.
    By Definition \ref{def:CST}, we can see that case (i) cannot apply since $\alpha_1 \in \Atts'(\alpha_2)$.
    Then, for case (ii), it can be seen from Definition \ref{def:Tnorm} that if $\tau'(\alpha_1) = 0$, then $\SF^{\graph'}_{T_m}(\alpha_2) = \SF^{\graph}_{T_m}(\alpha_2)$, while if $\tau'(\alpha_1) >0$, then $\SF^{\graph'}_{T_m}(\alpha_2) < \SF^{\graph}_{T_m}(\alpha_2)$. Therefore, $\alpha_1$ can only reduce $\SF^{\graph'}_{T_m}(\alpha_2)$, i.e., $\SF^{\graph'}_{T_m}(\alpha_2) \leq \SF^{\graph}_{T_m}(\alpha_2)$.

    \item[] Property \ref{prop:supportedpremise}: Supported Premise. It can be seen from Definition \ref{def:Tnorm} that there are two ways in which the addition of $\alpha_1$ may affect $\SF^{\graph'}_{T_m}(\alpha_2)$: (i) $\exists T_+ \in \CST^\graph(\alpha_2)$ where $\alpha_1 \in T_+$; and (ii) $\exists \alpha_3 \in \Args' \setminus \{ \alpha_2 \}$ such that $\exists T_- \in \CST^\graph(\alpha_3)$ where $T_-$ attacks some $T \in \CST^{\graph'}(\alpha_2)$ and $\alpha_1 \in T_-$.
    By Definition \ref{def:CST}, we can see that case (ii) cannot apply since $\alpha_1 \in \Supps'(\alpha_2)$.
    Then, for case (i), it can be seen from Definition \ref{def:Tnorm} that $\alpha_1$, regardless of the value of $\SF^{\graph'}_{T_m}(\alpha_1)$, can only increase $\SF^{\graph'}_{T_m}(\alpha_2)$, and so we have that $\SF^{\graph'}_{T_m}(\alpha_2) \geq \SF^{\graph}_{T_m}(\alpha_2)$.

    \item[] Property \ref{prop:weakenedpremise}: Weakened Premise (Counterexample). Let 
    $\Args = \{ \alpha_1, \alpha_2,\alpha_3, \alpha_4, \alpha_5, \alpha_6, \alpha_7, \alpha_8 \}$, where 
    $\alpha_1 = \langle a\wedge b, c \rangle$, 
    $\alpha_2 = \langle \top, a \rangle$, 
    $\alpha_3 = \langle \top, b \rangle $, 
    $\alpha_4 = \langle d, e \rangle $, 
    $\alpha_5 = \langle e, \neg b \rangle$, 
    $\alpha_6 = \langle \top, d \rangle$, 
    $\alpha_7 = \langle f, \neg d \rangle$, and 
    $\alpha_8 = \langle \top, f \rangle$. 
    Then, let 
    $\Args' = \Args$ and $\tau(\alpha_1) = \tau'(\alpha_1) = 
    \tau(\alpha_2) = \tau'(\alpha_2) = 
    \tau(\alpha_3) = \tau'(\alpha_3) = 
    \tau(\alpha_4) = 
    \tau(\alpha_5) = 
    \tau(\alpha_6) = 0.5$,
    $\tau'(\alpha_4) = 
    \tau'(\alpha_5) = 
    \tau'(\alpha_6) = 0.4$,
    $\tau(\alpha_7) = \tau(\alpha_8) = 1$ and $\tau'(\alpha_7) = \tau'(\alpha_8) = 0.1$. 
    Then, by Definition~\ref{def:Tnorm}, we have that $\SF^\graph_{T_m}(\alpha_1) = 0.25$ and $\SF^\graph_{T_m}(\alpha_5) = 0$. However, from the same definition, we get $\SF^{\graph'}_{T_m}(\alpha_1) = 0.3$ and $\SF^{\graph'}_{T_m}(\alpha_5) = 0.36$, thus $\SF^{\graph'}_{T_m}(\alpha_1) > \SF^\graph_{T_m}(\alpha_1)$.

    \item[] Property \ref{prop:strengthenedpremise}: Strengthened Premise.
    It follows directly from Definition~\ref{def:Tnorm} that increasing the strength of $\alpha_3 \in \Supps(\alpha_1)$, i.e., $\SF^{\graph'}_{T_m} (\alpha_3) > \SF^{\graph}_{T_m} (\alpha_3)$ while $\forall \alpha_2 \in \Atts(\alpha_1) \cup \Supps(\alpha_1)$, $\SF^{\graph'}_{T_m} (\alpha_2) = \SF^{\graph}_{T_m} (\alpha_2)$, then the strength of $\alpha_1$ can only increase, and so we have that $\SF^{\graph'}_{T_m} (\alpha_1) \geq \SF^{\graph}_{T_m} (\alpha_1)$.

    \item[] Property \ref{prop:bottomstrengthpremise}: Bottom-Strength Premise. 
    Let us prove by contradiction. 
    If we let $\SF^\graph_{T_m}(\alpha_1) > 0$, we can see from Definition \ref{def:Tnorm} that $\exists T_+ \in \CST^\graph(\alpha_1)$ where $\mathcal{O}(T_+) > 0$. However, since $\alpha_2 \in \Atts(\alpha_1)$ and $\SF^\graph_{T_m}(\alpha_2) = 1$, we can see from Definition \ref{def:Tnorm} that $\exists T_- \in \CST^\graph(\alpha_2)$ where $\mathcal{I}(T_-) = 1$ and exists $\alpha_3 \in T_-$ such that $\BS(\alpha_3) = 1$. 
    Thus, any $T_+$ must necessarily be attacked by $T_-$ and thus $\mathcal{O}(T_+) = 0$ and we have the contradiction.

    \item[] Property 15: Top-Strength Premises (Counterexample). Let $\Args = \{ \alpha_1, \alpha_2 \}$, where $\alpha_1 = \langle a, b \rangle $ and $\alpha_2 = \langle \top, a \rangle $, and $\tau(\alpha_1) = 0.8$ and $\tau(\alpha_2)=1$. Then, by Definition~\ref{def:Tnorm}, we have that $\SF^\graph_{T_m} (\alpha_2) = 1$, but $\SF^\graph_{T_m}(\alpha_1) = 0.8$.

    \item[] Property 16: Mirroring (Counterexample). Let $\Args = \{ \alpha_1, \alpha_2,\alpha_3, \alpha_4, \alpha_5 \}$, where $\alpha_1 = \langle a, b \rangle$, $\alpha_2 = \langle \top, a \rangle$, $\alpha_3 = \langle \neg a, c \rangle $, $\alpha_4 = \langle d, \neg a \rangle $, and $\alpha_5 = \langle \top, d \rangle $. Also, let $\tau(\alpha_1) =\tau(\alpha_2)= \tau(\alpha_3) = \tau(\alpha_4) = \tau(\alpha_5) = 0.5$. Then, by Definition~\ref{def:Tnorm}, we have that $\SF^\graph_{T_m}(\alpha_1) = 0.25$ and $\SF^\graph_{T_m} (\alpha_3) = 0.25$.
    
    \squishend
    
\end{proof}

\begin{dummyproposition}
    $\SF_{\wedge_D}$ satisfies Properties \ref{prop:directionality} and \ref{prop:weakprovability}-\ref{prop:mirroring}, but violates Properties \ref{prop:rewriting}-\ref{prop:provability}. %and \ref{prop:attackreinforcement}-\ref{prop:supportmonotonicity}.
     $\SF_{\wedge_Q}$ satisfies Properties \ref{prop:directionality}, \ref{prop:weakprovability}-\ref{prop:strengthenedpremise}, and \ref{prop:mirroring}, but violates Properties \ref{prop:rewriting}-\ref{prop:provability} and \ref{prop:bottomstrengthpremise}-\ref{prop:topstrengthpremises}.
\end{dummyproposition}

\begin{proof}

    $\SF_{\wedge_D}$:

    \squishlist

    \item[] Property \ref{prop:directionality}: Directionality. 
    It can be seen from Definitions \ref{def:DC}, \ref{def:modular} and \ref{def:DFQuAD} that $\SF^{\graph'}_{\wedge_D}(\alpha_2)$ depends only on $\BS'(\alpha_2)$ and $\{ \SF^{\graph'}_{\wedge_D}(\alpha_4)  | \alpha_4 \in \Atts'(\alpha_2) \cup \Supps'(\alpha_2) \}$.
    Thus, it cannot be the case that $\SF^{\graph'}_{\wedge_D}(\alpha_2) \neq \SF^{\graph}_{\wedge_D}(\alpha_2)$, since there is no directed path from $\alpha_1$ to $\alpha_2$.

    \item[] Property \ref{prop:rewriting}: Rewriting (Counterexample).
    Let $\Args = \{ \alpha_1, \alpha_2, \alpha_3 \}$ where 
    $\alpha_1 = \langle b \wedge c , a \rangle$, 
    $\alpha_2 = \langle b , d \rangle$ and
    $\alpha_3 = \langle c \wedge d, a \rangle$.
    By Definition \ref{def:supp-att}, we have that $\Atts = \emptyset$ and $\Supps = \{ (\alpha_2,\alpha_3) \}$.
    Then let $\BS(\alpha_1) = \BS(\alpha_3) = 0.5$ and 
    $\BS(\alpha_2) = 1$.
    By Definitions \ref{def:DC}, \ref{def:modular} and \ref{def:DFQuAD}, this gives 
    $\SF^{\graph}_{\wedge_D}(\alpha_1) = 0.5$, 
    $\SF^{\graph}_{\wedge_D}(\alpha_2) = 1$ and 
    $\SF^{\graph}_{\wedge_D}(\alpha_3) = 0.71$.
    Thus, $\SF^{\graph}_{\wedge_D}(\alpha_1) \neq \SF^{\graph}_{\wedge_D}(\alpha_3)$.
    
    \item[] Property \ref{prop:provability}: Provability (Counterexample).
    Let $\Args = \{ \alpha_1 \}$ where 
    $\alpha_1 = \langle b , a \rangle$.
    By Definition \ref{def:supp-att}, we have that $\Atts = \Supps = \emptyset$.
    Then let $\BS(\alpha_1) = 0.5$.
    By Definitions \ref{def:DC}, \ref{def:modular} and \ref{def:DFQuAD}, this gives $\SF^{\graph}_{\wedge_D}(\alpha_1) = 0.5$.

    \item[] Property \ref{prop:weakprovability}: Weak Provability.
    Let us prove by contradiction. If $\exists \alpha_1 \in \Args$ where $\SF^\graph_{\wedge_D}(\alpha_1) \neq 0$, then, by Definition \ref{def:DC}, $\forall x \in Prem(\alpha_1)$, $\SF^\graph_{D}(x) > 0$. However, since, by the same definition, $\forall x \in Prem(\alpha_1)$, $\BS(x) = \sqrt[n]{\BS(\alpha_1)}$, where $n = |Prem(\alpha_1)|$ and $\BS(\alpha_1) = 0$, we know that for any $n$, $\BS(x) = 0$. Then, since $\SF_D$ satisfies balance \cite[Proposition 37]{Baroni_19}, for $\SF^\graph_{D}(x) > 0$ to hold it must be the case that $\exists \langle \Phi_2, \Psi_2 \rangle \in \Supps(\alpha_1)$ s.t. $x \equiv \Psi_2$, by Definition \ref{def:supp-att}, and thus we have the contradiction.

    \item[] Property \ref{prop:stability}: Stability. By Definition \ref{def:DC}, we can see that if $Prem(\alpha) = \emptyset$, then $\SF^\graph_{\wedge_D}(\alpha) = \BS(\alpha)$. Otherwise, by the same definition $\BS(x) = \sqrt[n]{\BS(\alpha)}$ $\forall x \in Prem(\alpha) =  \{x_1, \ldots, x_n \}$. Then, by the same definition and Definition \ref{def:DFQuAD}, it can be seen that $\SF^\graph_D(x) = \BS(x)$ since $\Atts(x) = \Supps(x) = \emptyset$ and thus, it follows that $\SF^\graph_{\wedge_D}(\alpha) = \SF^\graph_D(x_1) \times \ldots \times \SF^\graph_D(x_n) = \BS(\alpha)$.
    
    \item[] Property \ref{prop:neutrality}: Neutrality. It can be seen from Definition \ref{def:DC} that $\SF^{\graph'}_{\wedge_D}(\alpha_2)$ is a product of $\SF^{\graph'}_{D}(x)$ $\forall x \in Prem(\alpha_2)$. However, by Definition \ref{def:DFQuAD}, $\alpha_1$ has no impact on either $\Sigma(\SF^\graph_D(\Atts(x)))$ or $\Sigma(\SF^\graph_D(\Supps(x)))$ $\forall x \in Prem(\alpha_2)$, and thus similarly for $\SF^\graph_D(x)$. It follows that $\SF^{\graph'}_{\wedge_D}(\alpha_2) = \SF^\graph_{\wedge_D}(\alpha_2)$.

    \item[] Property \ref{prop:attackedpremise}: Attacked Premise. It can be seen from Definition \ref{def:DFQuAD} that for any $x \in Prem(\alpha_1)$, it can only be the case that $\Sigma(\SF^{\graph'}_D(\Atts(x))) \geq \Sigma(\SF^\graph_D(\Atts(x)))$ and $\Sigma(\SF^{\graph'}_D(\Supps(x))) = \Sigma(\SF^\graph_D(\Supps(x)))$. It follows from the same definition that, regardless of the value of $\BS'(x) = \BS(x)$, $\SF^{\graph'}_D(x) \leq \SF^\graph_D(x)$. Then, by Definition \ref{def:DC}, $\SF^{\graph'}_{\wedge_D}(\alpha_1) \leq \SF^\graph_{\wedge_D}(\alpha_1)$.

    \item[] Property \ref{prop:supportedpremise}: Supported Premise. It can be seen from Definition \ref{def:DFQuAD} that for any $x \in Prem(\alpha_1)$, it can only be the case that $\Sigma(\SF^{\graph'}_D(\Atts(x))) = \Sigma(\SF^\graph_D(\Atts(x)))$ and $\Sigma(\SF^{\graph'}_D(\Supps(x))) \geq \Sigma(\SF^\graph_D(\Supps(x)))$. It follows from the same definition that, regardless of the value of $\BS'(x) = \BS(x)$, $\SF^{\graph'}_D(x) \geq \SF^\graph_D(x)$. Then, by Definition \ref{def:DC}, $\SF^{\graph'}_{\wedge_D}(\alpha_1) \geq \SF^\graph_{\wedge_D}(\alpha_1)$.

    \item[] Property \ref{prop:weakenedpremise}: Weakened Premise. It can be seen from Definition \ref{def:DFQuAD} that for any $x \in Prem(\alpha_2)$, it can only be the case that $\Sigma(\SF^{\graph'}_D(\Atts(x))) \geq \Sigma(\SF^\graph_D(\Atts(x)))$ and $\Sigma(\SF^{\graph'}_D(\Supps(x))) = \Sigma(\SF^\graph_D(\Supps(x)))$. It follows from the same definition that, regardless of the value of $\BS'(x) = \BS(x)$, $\SF^{\graph'}_D(x) \leq \SF^\graph_D(x)$. Then, by Definition \ref{def:DC}, $\SF^{\graph'}_{\wedge_D}(\alpha_2) \leq \SF^\graph_{\wedge_D}(\alpha_2)$.

    \item[] Property \ref{prop:strengthenedpremise}: Strengthened Premise. It can be seen from Definition \ref{def:DFQuAD} that for any $x \in Prem(\alpha_2)$, it can only be the case that $\Sigma(\SF^{\graph'}_D(\Atts(x))) = \Sigma(\SF^\graph_D(\Atts(x)))$ and $\Sigma(\SF^{\graph'}_D(\Supps(x))) \geq \Sigma(\SF^\graph_D(\Supps(x)))$. It follows from the same definition that, regardless of the value of $\BS'(x) = \BS(x)$, $\SF^{\graph'}_D(x) \geq \SF^\graph_D(x)$. Then, by Definition \ref{def:DC}, $\SF^{\graph'}_{\wedge_D}(\alpha_2) \geq \SF^\graph_{\wedge_D}(\alpha_2)$.

    \item[] Property \ref{prop:bottomstrengthpremise}: Bottom-Strength Premise. It can be seen from Definition \ref{def:DFQuAD} that $\Sigma(\SF^\graph_D(\Atts(x))) = 1$ while $\Sigma(\SF^\graph_D(\Supps(x))) = 0$. It follows from the same definition that, regardless of $\BS(x)$, $\SF^\graph_D(x) = 0$. Then, by Definition \ref{def:DC}, $\SF^\graph_{\wedge_D}(\alpha_1) = 0$.

    \item[] Property \ref{prop:topstrengthpremises}: Top-Strength Premises. It can be seen from Definition \ref{def:DFQuAD} that $\forall x \in Prem(\alpha_1)$, $\Sigma(\SF^\graph_D(\Atts(x))) = 0$ while $\Sigma(\SF^\graph_D(\Supps(x))) = 1$. It follows from the same definition that, regardless of $\BS(x)$, $\SF^\graph_D(x) = 1$. Then, by Definition \ref{def:DC}, $\SF^\graph_{\wedge_D}(\alpha_1) = 1$.

    \item[] Property \ref{prop:mirroring}: Mirroring. Let the single premises for each statement be $x_1 \in Prem(\alpha_1)$ and $x_2 \in Prem(\alpha_2)$. 
    Then, it can be seen from Definitions \ref{def:supp-att} and \ref{def:graph} that these conditions mean that $\Atts(x_1) = \Supps(x_2)$ and $\Atts(x_2) = \Supps(x_1)$. 
    Thus, by Definition \ref{def:DFQuAD}, letting $v_1^- = \Sigma(\SF^\graph_D(\Atts(x_1)))$, $v_1^+ = \Sigma(\SF^\graph_D(\Supps(x_1)))$, $v_2^- = \Sigma(\SF^\graph_D(\Atts(x_2)))$ and $v_2^+ = \Sigma(\SF^\graph_D(\Supps(x_2)))$, we have that $v_1^- = v_2^+$ and $v_1^+ = v_2^-$.
    We can also see from Definition \ref{def:DC} that $\BS(x_1) = \BS(x_2) = \sqrt[1]{\BS(\alpha_1)}  = \sqrt[1]{\BS(\alpha_2)} = 0.5$.
    Then, by Definition \ref{def:DFQuAD}, we have three cases:
    i.) $v_1^- = v_2^+ = v_1^+ = v_2^-$, giving $\SF^\graph_D(x_1) = 0.5 - 0.5\cdot| v_1^+ - v_1^- | = 0.5$ and $\SF^\graph_D(x_1) = 0.5 - 0.5\cdot| v_2^+ - v_2^- | = 0.5$;
    ii.) $v_1^- = v_2^+ > v_1^+ = v_2^-$, giving $\SF^\graph_D(x_1) = 0.5 - 0.5\cdot| v_1^+ - v_1^- |$ and $\SF^\graph_D(x_2) = 0.5 + (1- 0.5)\cdot| v_2^+ - v_2^-|$;
    iii.) $v_1^- = v_2^+ < v_1^+ = v_2^-$, giving $\SF^\graph_D(x_1) = 0.5 + (1- 0.5)\cdot| v_1^+ - v_1^-|$ and $\SF^\graph_D(x_2) = 0.5 - 0.5\cdot| v_2^+ - v_2^- |$.
    Note that $| v_1^+ - v_1^-| = | v_2^+ - v_2^- |$ for all cases.
    Also note that $\SF^\graph_{\wedge_D}(\alpha_1) = \SF^\graph_D(x_1)$ and $\SF^\graph_{\wedge_D}(\alpha_2) = \SF^\graph_D(x_2)$, by Definition \ref{def:DC}, since $n=1$ for both statements.
    Thus, in all three cases, it holds that $\SF^\graph_{\wedge_D}(\alpha_1) = 1 - \SF^\graph_{\wedge_D}(\alpha_2)$.

    \squishend

    $\SF_{\wedge_Q}$:

    \squishlist

    \item[] Property \ref{prop:directionality}: Directionality. 
    It can be seen from Definitions \ref{def:DC}, \ref{def:modular} and \ref{def:QEM} that $\SF^{\graph'}_{\wedge_Q}(\alpha_2)$ depends only on $\BS'(\alpha_2)$ and $\{ \SF^{\graph'}_{\wedge_Q}(\alpha_4)  | \alpha_4 \in \Atts'(\alpha_2) \cup \Supps'(\alpha_2) \}$.
    Thus, it cannot be the case that $\SF^{\graph'}_{\wedge_Q}(\alpha_2) \neq \SF^{\graph}_{\wedge_Q}(\alpha_2)$, since there is no directed path from $\alpha_1$ to $\alpha_2$. 
    
    \item[] Property \ref{prop:rewriting}: Rewriting (Counterexample).
    Let $\Args = \{ \alpha_1, \alpha_2, \alpha_3 \}$ where 
    $\alpha_1 = \langle b \wedge c , a \rangle$, 
    $\alpha_2 = \langle b , d \rangle$ and
    $\alpha_3 = \langle c \wedge d, a \rangle$.
    By Definition \ref{def:supp-att}, we have that $\Atts = \emptyset$ and $\Supps = \{ (\alpha_2,\alpha_3) \}$.
    Then let $\BS(\alpha_1) = \BS(\alpha_3) = 0.5$ and 
    $\BS(\alpha_2) = 1$.
    By Definitions \ref{def:DC}, \ref{def:modular} and \ref{def:QEM}, this gives 
    $\SF^{\graph}_{\wedge_Q}(\alpha_1) = 0.5$, 
    $\SF^{\graph}_{\wedge_Q}(\alpha_2) = 1$ and 
    $\SF^{\graph}_{\wedge_Q}(\alpha_3) = 0.6$.
    Thus, $\SF^{\graph}_{\wedge_Q}(\alpha_1) \neq \SF^{\graph}_{\wedge_Q}(\alpha_3)$.
    
    \item[] Property \ref{prop:provability}: Provability (Counterexample).
    Let $\Args = \{ \alpha_1 \}$ where 
    $\alpha_1 = \langle b , a \rangle$.
    By Definition \ref{def:supp-att}, we have that $\Atts = \Supps = \emptyset$.
    Then let $\BS(\alpha_1) = 0.5$.
    By Definitions \ref{def:DC}, \ref{def:modular} and \ref{def:QEM}, this gives $\SF^{\graph}_{\wedge_Q}(\alpha_1) = 0.5$.

    \item[] Property \ref{prop:weakprovability}: Weak Provability.
    Let us prove by contradiction. If $\exists \alpha_1 \in \Args$ where $\SF^\graph_{\wedge_Q}(\alpha_1) \neq 0$, then, by Definition \ref{def:DC}, $\forall x \in Prem(\alpha_1)$, $\SF^\graph_{Q}(x) > 0$. However, since, by the same definition, $\forall x \in Prem(\alpha_1)$, $\BS(x) = \sqrt[n]{\BS(\alpha_1)}$, where $n = |Prem(\alpha_1)|$ and $\BS(\alpha_1) = 0$, we know that for any $n$, $\BS(x) = 0$. Then, since $\SF_Q$ satisfies weakening and strengthening \cite[Propositions 13 and 14]{Potyka_18}, for $\SF^\graph_{Q}(x) > 0$ to hold it must be the case that $\exists \langle \Phi_2, \Psi_2 \rangle \in \Supps(\alpha_1)$ s.t. $x \equiv \Psi_2$, by Definition \ref{def:supp-att}, and thus we have the contradiction.

    \item[] Property \ref{prop:stability}: Stability. By Definition \ref{def:DC}, we can see that if $Prem(\alpha) = \emptyset$, then $\SF^\graph_{\wedge_Q}(\alpha) = \BS(\alpha)$. Otherwise, by the same definition, $\BS(x) = \sqrt[n]{\BS(\alpha)}$ $\forall x \in Prem(\alpha) =  \{x_1, \ldots, x_n \}$. Then, by the same definition and Definition \ref{def:QEM}, it can be seen that $\SF^\graph_Q(x) = \BS(x)$ since $\Atts(x) = \Supps(x) = \emptyset$ and thus, it follows that $\SF^\graph_{\wedge_Q}(\alpha) = \SF^\graph_Q(x_1) \times \ldots \times \SF^\graph_Q(x_n) = \BS(\alpha)$.
    
    \item[] Property \ref{prop:neutrality}: Neutrality. It can be seen from Definition \ref{def:DC} that $\SF^{\graph'}_{\wedge_Q}(\alpha_2)$ is a product of $\SF^{\graph'}_{Q}(x)$ $\forall x \in Prem(\alpha_2)$. However, by Definition \ref{def:QEM}, $\alpha_1$ has no impact on $E_x$ and thus $\SF^\graph_Q(x)$ $\forall x \in Prem(\alpha_2)$. It follows that $\SF^{\graph'}_{\wedge_Q}(\alpha_2) = \SF^\graph_{\wedge_Q}(\alpha_2)$.

    \item[] Property \ref{prop:attackedpremise}: Attacked Premise. It can be seen from Definition \ref{def:QEM} that for any $x \in Prem(\alpha_2)$, it can only be the case that $\Sigma_{x_- \in \Atts'(x)}\SF^{\graph'}_Q(x_-) > \Sigma_{x_- \in \Atts(x)}\SF^{\graph}_Q(x_-)$ and $\Sigma_{x_+ \in \Supps'(x)}\SF^{\graph'}_Q(x_+) = \Sigma_{x_+ \in \Supps(x)}\SF^{\graph}_Q(x_+)$. It follows from the same definition that $E'_x < E_x$ and, regardless of the value of $\BS'(x) = \BS(x)$, $\SF^{\graph'}_Q(x) \leq \SF^\graph_Q(x)$. Then, by Definition \ref{def:DC}, $\SF^{\graph'}_{\wedge_Q}(\alpha_2) \leq \SF^\graph_{\wedge_Q}(\alpha_2)$.

    \item[] Property \ref{prop:supportedpremise}: Supported Premise. It can be seen from Definition \ref{def:QEM} that for any $x \in Prem(\alpha_2)$, it can only be the case that $\Sigma_{x_- \in \Atts'(x)}\SF^{\graph'}_Q(x_-) = \Sigma_{x_- \in \Atts(x)}\SF^{\graph}_Q(x_-)$ and $\Sigma_{x_+ \in \Supps'(x)}\SF^{\graph'}_Q(x_+) > \Sigma_{x_+ \in \Supps(x)}\SF^{\graph}_Q(x_+)$. It follows from the same definition that $E'_x > E_x$ and, regardless of the value of $\BS'(x) = \BS(x)$, $\SF^{\graph'}_Q(x) \geq \SF^\graph_Q(x)$. Then, by Definition \ref{def:DC}, $\SF^{\graph'}_{\wedge_Q}(\alpha_2) \geq \SF^\graph_{\wedge_Q}(\alpha_2)$.

    \item[] Property \ref{prop:weakenedpremise}: Weakened Premise. It can be seen from Definition \ref{def:QEM} that for any $x \in Prem(\alpha_1)$, it can only be the case that $\Sigma_{x_- \in \Atts'(x)}\SF^{\graph'}_Q(x_-) > \Sigma_{x_- \in \Atts(x)}\SF^{\graph}_Q(x_-)$ and $\Sigma_{x_+ \in \Supps'(x)}\SF^{\graph'}_Q(x_+) = \Sigma_{x_+ \in \Supps(x)}\SF^{\graph}_Q(x_+)$. It follows from the same definition that $E'_x < E_x$ and, regardless of the value of $\BS'(x) = \BS(x)$, $\SF^{\graph'}_Q(x) \leq \SF^\graph_Q(x)$. Then, by Definition \ref{def:DC}, $\SF^{\graph'}_{\wedge_Q}(\alpha_1) \leq \SF^\graph_{\wedge_Q}(\alpha_1)$.

    \item[] Property \ref{prop:strengthenedpremise}: Strengthened Premise. It can be seen from Definition \ref{def:QEM} that for any $x \in Prem(\alpha_1)$, it can only be the case that $\Sigma_{x_- \in \Atts'(x)}\SF^{\graph'}_Q(x_-) = \Sigma_{x_- \in \Atts(x)}\SF^{\graph}_Q(x_-)$ and $\Sigma_{x_+ \in \Supps'(x)}\SF^{\graph'}_Q(x_+) > \Sigma_{x_+ \in \Supps(x)}\SF^{\graph}_Q(x_+)$. It follows from the same definition that $E'_x > E_x$ and, regardless of the value of $\BS'(x) = \BS(x)$, $\SF^{\graph'}_Q(x) \geq \SF^\graph_Q(x)$. Then, by Definition \ref{def:DC}, $\SF^{\graph'}_{\wedge_Q}(\alpha_1) \geq \SF^\graph_{\wedge_Q}(\alpha_1)$.

    \item[] Property \ref{prop:bottomstrengthpremise}: Bottom-Strength Premise (Counterexample).
    Let $\Args = \{ \alpha_1, \alpha_2 \}$ where 
    $\alpha_1 = \langle b , a \rangle$ and $\alpha_2 = \langle \top , \neg b \rangle$.
    By Definition \ref{def:supp-att}, we have that $\Atts = \{ (\alpha_2, \alpha_1) \}$ and $\Supps =  \emptyset$.
    Then let $\BS(\alpha_1) = 0.5$ and $\BS(\alpha_2) = 1$.
    By Definitions \ref{def:DC}, \ref{def:modular} and \ref{def:QEM}, this gives $\SF^{\graph}_{\wedge_Q}(\alpha_2) = 1$ and $\SF^{\graph}_{\wedge_Q}(\alpha_1) = 0.25$. 

    \item[] Property \ref{prop:topstrengthpremises}: Top-Strength Premises (Counterexample).
    Let $\Args = \{ \alpha_1, \alpha_2 \}$ where 
    $\alpha_1 = \langle b , a \rangle$ and $\alpha_2 = \langle \top , b \rangle$.
    By Definition \ref{def:supp-att}, we have that $\Atts = \emptyset$ and $\Supps = \{ (\alpha_2, \alpha_1) \}$.
    Then let $\BS(\alpha_1) = 0.5$ and $\BS(\alpha_2) = 1$.
    By Definitions \ref{def:DC}, \ref{def:modular} and \ref{def:QEM}, this gives $\SF^{\graph}_{\wedge_Q}(\alpha_2) = 1$ and $\SF^{\graph}_{\wedge_Q}(\alpha_1) = 0.75$.

    \item[] Property \ref{prop:mirroring}: Mirroring. Let the single premises for each statement be $x_1 \in Prem(\alpha_1)$ and $x_2 \in Prem(\alpha_2)$. 
    Then, it can be seen from Definitions \ref{def:supp-att} and \ref{def:graph} that these conditions mean that $\Atts(x_1) = \Supps(x_2)$ and $\Atts(x_2) = \Supps(x_1)$. 
    Thus, by Definition \ref{def:QEM}, $E_{x_1} = -E_{x_2}$ and then $h(E_{x_1}) = -h(E_{x_2})$.
    We can also see from Definition \ref{def:DC} that $\BS(x_1) = \BS(x_2) = \sqrt[1]{\BS(\alpha_1)}  = \sqrt[1]{\BS(\alpha_2)} = 0.5$.
    Then, again by Definition \ref{def:QEM}, $\SF^\graph_Q(\alpha_1) = 0.5 + (1 - 0.5) \cdot h(E_{\alpha_1}) -  0.5 \cdot h(-E_{\alpha_1})$ and $\SF^\graph_{\wedge_Q}(\alpha_2) = 0.5 + (1 - 0.5) \cdot h(E_{\alpha_2}) -  0.5 \cdot h(-E_{\alpha_2})$ and thus it holds that $\SF^\graph_{\wedge_Q}(\alpha_1) = 1 - \SF^\graph_{\wedge_Q}(\alpha_2)$.

    \squishend
\end{proof}

\begin{dummyproposition}
    Properties \ref{prop:rewriting}, \ref{prop:stability} and \ref{prop:topstrengthpremises} are incompatible.
\end{dummyproposition}

\begin{proof}

    Let us prove that Properties \ref{prop:rewriting}, \ref{prop:stability} and \ref{prop:topstrengthpremises} are incompatible by contradiction.
    Let us assume the existence of a semantics $\SF_a$ which satisfies all three properties.
    Consider $\Args = \{ \alpha_1, \alpha_2, \alpha_3 \}$ where 
    $\alpha_1 = \langle b , a \rangle$, 
    $\alpha_2 = \langle b , c \rangle$ and
    $\alpha_3 = \langle c, a \rangle$.
    By Definition \ref{def:supp-att}, we have that $\Atts = \emptyset$ and $\Supps = \{ (\alpha_2,\alpha_3) \}$.
    Then let $\BS(\alpha_1) = \BS(\alpha_3) = 0.5$ and 
    $\BS(\alpha_2) = 1$.
    Then, Property \ref{prop:stability} requires that $\SF^{\graph}_{a}(\alpha_1) = 0.5$, Property \ref{prop:rewriting} requires that $\SF^{\graph}_{a}(\alpha_1) = \SF^{\graph}_{a}(\alpha_2)$, 
    and Property \ref{prop:topstrengthpremises} requires that $\SF^{\graph}_{a}(\alpha_1) = 0.5$, and so we have the contradiction.

\end{proof}

\begin{dummyproposition}
 Properties \ref{prop:provability} and \ref{prop:stability} are incompatible.
\end{dummyproposition}

\begin{proof}

    Let us prove that Properties \ref{prop:provability} and \ref{prop:stability} are incompatible by contradiction.
    Let us assume the existence of a semantics $\SF_a$ which satisfies both properties.
    Consider some $\alpha_1 \in \Args$ such that $\Atts(\alpha_1) = \Atts(\alpha_1) = \emptyset$ and $\BS(\alpha_1) \neq 0$.
    Property \ref{prop:provability} requires that $\SF_a^\graph(\alpha_1) \neq 0 = \BS(\alpha_1)$, but Property \ref{prop:provability} requires that $\SF_a^\graph(\alpha_1) = 0$, and we have the contradiction.

\end{proof}

\begin{dummyproposition}
    $\SF_{D}$ and $\SF_{Q}$ satisfy Properties \ref{prop:directionality} and \ref{prop:stability}-\ref{prop:strengthenedpremise}.
    %, but violates \note{x,y,z}.
    (They are incompatible with the other properties.)
\end{dummyproposition}

\begin{proof}
    $\SF_{D}$:

    \squishlist
    \item[]Property \ref{prop:directionality}: Directionality. 
    It can be seen from Definition \ref{def:DFQuAD} that $\SF^{\graph'}_{D}(\alpha_2)$ depends only on $\BS'(\alpha_2)$ and $\{ \SF^{\graph'}_{D}(\alpha_4)  | \alpha_4 \in \Atts'(\alpha_2) \cup \Supps'(\alpha_2) \}$.
    Thus, it cannot be the case that $\SF^{\graph'}_{D}(\alpha_2) \neq \SF^{\graph}_{D}(\alpha_2)$, since there is no directed path from $\alpha_1$ to $\alpha_2$.

    \item[]Property \ref{prop:stability}: Stability. By Definition \ref{def:DFQuAD}, it can be seen that $\SF^\graph_{D}(\alpha) = \BS(\alpha)$ since $\Atts(\alpha) = \Supps(\alpha) = \emptyset$.

    \item[]Property \ref{prop:neutrality}: Neutrality. It can be seen from Definition \ref{def:DFQuAD} that $\alpha_1$ has no impact on either $\Sigma(\SF^\graph_D(\Atts(\alpha_2)))$ or $\Sigma(\SF^\graph_D(\Supps(\alpha_2)))$, and thus similarly for $\SF^\graph_D(\alpha_2)$. It follows that $\SF^{\graph'}_{D}(\alpha_2) = \SF^\graph_{D}(\alpha_2)$.

    \item[]Property \ref{prop:attackedpremise}: Attacked Premise. It can be seen from Definition \ref{def:DFQuAD} that it can only be the case that $\Sigma(\SF^{\graph'}_D(\Atts(\alpha_2))) \geq \Sigma(\SF^\graph_D(\Atts(\alpha_2)))$ and $\Sigma(\SF^{\graph'}_D(\Supps(\alpha_2))) = \Sigma(\SF^\graph_D(\Supps(\alpha_2)))$. It follows from the same definition that, regardless of the value of $\BS'(\alpha_2) = \BS(\alpha_2)$, $\SF^{\graph'}_{D}(\alpha_2) \leq \SF^\graph_{D}(\alpha_2)$.

    \item[]Property \ref{prop:supportedpremise}: Supported Premise. It can be seen from Definition \ref{def:DFQuAD} that it can only be the case that $\Sigma(\SF^{\graph'}_D(\Atts(\alpha_2))) = \Sigma(\SF^\graph_D(\Atts(\alpha_2)))$ and $\Sigma(\SF^{\graph'}_D(\Supps(\alpha_2))) \geq \Sigma(\SF^\graph_D(\Supps(\alpha_2)))$. It follows from the same definition that, regardless of the value of $\BS'(\alpha_2) = \BS(\alpha_2)$, $\SF^{\graph'}_{D}(\alpha_2) \geq \SF^\graph_{D}(\alpha_2)$.

    \item[]Property \ref{prop:weakenedpremise}: Weakened Premise. It can be seen from Definition \ref{def:DFQuAD} that it can only be the case that $\Sigma(\SF^{\graph'}_D(\Atts(\alpha_1))) \geq \Sigma(\SF^\graph_D(\Atts(\alpha_1)))$ and $\Sigma(\SF^{\graph'}_D(\Supps(\alpha_1))) = \Sigma(\SF^\graph_D(\Supps(\alpha_1)))$. It follows from the same definition that, regardless of the value of $\BS'(\alpha_1) = \BS(\alpha_1)$, $\SF^{\graph'}_{D}(\alpha_1) \leq \SF^\graph_{D}(\alpha_1)$.

    \item[]Property \ref{prop:strengthenedpremise}: Strengthened Premise. It can be seen from Definition \ref{def:DFQuAD} that it can only be the case that $\Sigma(\SF^{\graph'}_D(\Atts(\alpha_1))) = \Sigma(\SF^\graph_D(\Atts(\alpha_1)))$ and $\Sigma(\SF^{\graph'}_D(\Supps(\alpha_1))) \geq \Sigma(\SF^\graph_D(\Supps(\alpha_1)))$. It follows from the same definition that, regardless of the value of $\BS'(\alpha_1) = \BS(\alpha_1)$, $\SF^{\graph'}_{D}(\alpha_1) \geq \SF^\graph_{D}(\alpha_1)$.

    \squishend
     
    $\SF_{Q}$:

    \squishlist
    \item[]Property \ref{prop:directionality}: Directionality. 
    It can be seen from Definition \ref{def:QEM} that $\SF^{\graph'}_{Q}(\alpha_2)$ depends only on $\BS'(\alpha_2)$ and $\{ \SF^{\graph'}_{Q}(\alpha_4)  | \alpha_4 \in \Atts'(\alpha_2) \cup \Supps'(\alpha_2) \}$.
    Thus, it cannot be the case that $\SF^{\graph'}_{Q}(\alpha_2) \neq \SF^{\graph}_{Q}(\alpha_2)$, since there is no directed path from $\alpha_1$ to $\alpha_2$.

    \item[]Property \ref{prop:stability}: Stability. By Definition \ref{def:QEM}, it can be seen that $\SF^\graph_{Q}(\alpha) = \BS(\alpha)$ since $\Atts(\alpha) = \Supps(\alpha) = \emptyset$.

    \item[]Property \ref{prop:neutrality}: Neutrality. It can be seen from Definition \ref{def:QEM} that $\alpha_1$ has no impact on $E_{\alpha_2}$ and thus it follows that $\SF^{\graph'}_{Q}(\alpha_2) = \SF^\graph_{Q}(\alpha_2)$.

    \item[]Property \ref{prop:attackedpremise}: Attacked Premise. It can be seen from Definition \ref{def:QEM} that it can only be the case that $\Sigma_{\alpha_3 \in \Atts'(\alpha_2)}\SF^{\graph'}_Q(\alpha_3) > \Sigma_{\alpha_3 \in \Atts(\alpha_2)}\SF^{\graph}_Q(\alpha_3)$ and $\Sigma_{\alpha_4 \in \Supps'(\alpha_2)}\SF^{\graph'}_Q(\alpha_4) = \Sigma_{\alpha_4 \in \Supps(\alpha_2)}\SF^{\graph}_Q(\alpha_4)$. It follows from the same definition that $E'_{\alpha_2} < E_{\alpha_2}$ and thus, regardless of the value of $\BS'(\alpha_2) = \BS(\alpha_2)$, $\SF^{\graph'}_{Q}(\alpha_2) \leq \SF^\graph_{Q}(\alpha_2)$.

    \item[]Property \ref{prop:supportedpremise}: Supported Premise. It can be seen from Definition \ref{def:QEM} that it can only be the case that $\Sigma_{\alpha_3 \in \Atts'(\alpha_2)}\SF^{\graph'}_Q(\alpha_3) = \Sigma_{\alpha_3 \in \Atts(\alpha_2)}\SF^{\graph}_Q(\alpha_3)$ and $\Sigma_{\alpha_4 \in \Supps'(\alpha_2)}\SF^{\graph'}_Q(\alpha_4) > \Sigma_{\alpha_4 \in \Supps(\alpha_2)}\SF^{\graph}_Q(\alpha_4)$. It follows from the same definition that $E'_{\alpha_2} > E_{\alpha_2}$ and thus, regardless of the value of $\BS'(\alpha_2) = \BS(\alpha_2)$, $\SF^{\graph'}_{Q}(\alpha_2) \geq \SF^\graph_{Q}(\alpha_2)$.
    
    \item[] Property \ref{prop:weakenedpremise}: Weakened Premise. It can be seen from Definition \ref{def:QEM} that it can only be the case that $\Sigma_{\alpha_2 \in \Atts'(\alpha_1)}\SF^{\graph'}_Q(\alpha_2) > \Sigma_{\alpha_2 \in \Atts(\alpha_1)}\SF^{\graph}_Q(\alpha_2)$ and $\Sigma_{\alpha_3 \in \Supps'(\alpha_1)}\SF^{\graph'}_Q(\alpha_3) = \Sigma_{\alpha_3 \in \Supps(\alpha_1)}\SF^{\graph}_Q(\alpha_3)$. It follows from the same definition that $E'_{\alpha_1} < E_{\alpha_1}$ and thus, regardless of the value of $\BS'(\alpha_1) = \BS(\alpha_1)$, $\SF^{\graph'}_{Q}(\alpha_1) \leq \SF^\graph_{Q}(\alpha_1)$.

    \item[] Property \ref{prop:strengthenedpremise}: Strengthened Premise. It can be seen from Definition \ref{def:QEM} that it can only be the case that $\Sigma_{\alpha_2 \in \Atts'(\alpha_1)}\SF^{\graph'}_Q(\alpha_2) = \Sigma_{\alpha_2 \in \Atts(\alpha_1)}\SF^{\graph}_Q(\alpha_2)$ and $\Sigma_{\alpha_3 \in \Supps'(\alpha_1)}\SF^{\graph'}_Q(\alpha_3) > \Sigma_{\alpha_3 \in \Supps(\alpha_1)}\SF^{\graph}_Q(\alpha_3)$. It follows from the same definition that $E'_{\alpha_1} > E_{\alpha_1}$ and thus, regardless of the value of $\BS'(\alpha_1) = \BS(\alpha_1)$, $\SF^{\graph'}_{Q}(\alpha_1) \geq \SF^\graph_{Q}(\alpha_1)$.
    \squishend
\end{proof}

\begin{dummytheorem}
    Let $\Args$ be such that $\forall \alpha \in \Args$, $|Prem(\alpha)| \leq 1$. Then, for any given abstract GS $\SF_i$, $\SF_{\wedge_i} = \SF_i$.
\end{dummytheorem}

\begin{proof}
    If we let $Prem(\alpha) = \{ x \}$, by Definition \ref{def:DC} it can be seen that $\BS(x) = \sqrt[1](\BS(\alpha)) = \BS(\alpha)$ and that $\SF^\graph_{\wedge_i}(\alpha) = \SF^\graph_{i}(x)$.
    Then, by Definitions \ref{def:modular} and \ref{def:supp-att}, $\Atts(x) = \Atts(\alpha)$ and $\Supps(x) = \Atts(\alpha)$.
    Thus, it must be the case that $\SF^\graph_{\wedge_i}(\alpha) = \SF^\graph_{i}(\alpha)$.
    Meanwhile, if we let $Prem(\alpha) = \emptyset$, by Definitions \ref{def:modular} and \ref{def:supp-att}, it can be seen that $\Atts(\alpha) = \Supps(\alpha) = \emptyset$. Then, by Definition \ref{def:DC}, $\SF^\graph_{\wedge_i}(\alpha) = \BS(\alpha) = \SF^\graph_{i}(\alpha)$.
\end{proof}

\begin{dummyproposition}
    $\SF_{T_p}$ and $\SF_{T_m}$ satisfy Properties \ref{prop:attackreinforcement}-\ref{prop:supportmonotonicity}.
\end{dummyproposition}

\begin{proof}

    $\SF_{T_p}$:

    \squishlist
    
    \item[] The proofs for Properties \ref{prop:attackreinforcement}-\ref{prop:supportmonotonicity} can be found in~\cite{Jedwabny_20}.

    \squishend
    
    $\SF_{T_m}$:

    \squishlist

    \item[] The proofs for Properties \ref{prop:attackreinforcement}-\ref{prop:supportmonotonicity} are analogous to those for $\SF_{T_p}$ in~\cite{Jedwabny_20}.    
    
    \squishend

\end{proof}

\begin{dummyproposition}
     $\SF_{\wedge_D}$ and $\SF_{\wedge_Q}$ violate Properties \ref{prop:attackreinforcement}-\ref{prop:supportmonotonicity}. 
\end{dummyproposition}

\begin{proof}

    $\SF_{\wedge_D}$:

    \squishlist

    \item[] Property \ref{prop:attackreinforcement}: Attack Reinforcement (Counterexample).
    Let $\Args = \{ \alpha_1, \alpha_2, \alpha_3, \alpha_4, \alpha_5 \}$ where 
    $\alpha_1 = \langle b , a \rangle$, 
    $\alpha_2 = \langle \top, b \rangle$, 
    $\alpha_3 = \langle  c \wedge d , b \rangle$, 
    $\alpha_4 = \langle c  , \neg b \rangle$ and 
    $\alpha_5 = \langle \top, c \rangle$.
    %\note{AR: confusingly, I think this may need another argument $\alpha_5 = \langle \top, b \rangle$ such that $\alpha_3$ is attacking a CST of $\alpha_1$, rather than just attacking $\alpha_1$?}
    By Definition \ref{def:supp-att}, we have that $\Atts = \{ (\alpha_4,\alpha_1) \}$ and $\Supps = \{ (\alpha_2,\alpha_1),(\alpha_3,\alpha_1), (\alpha_5,\alpha_3), (\alpha_5,\alpha_4) \}$.
    Then let $\BS(\alpha_1) = \BS'(\alpha_1) = \BS(\alpha_3) = \BS'(\alpha_3) = \BS(\alpha_5) = 0.5$, 
    $\BS(\alpha_2) = \BS'(\alpha_2) = 0$, 
    $\BS(\alpha_4) = \BS'(\alpha_4) = 1$ and $\BS'(\alpha_5) = 0.6$.
    By Definitions \ref{def:DC}, \ref{def:modular} and \ref{def:DFQuAD}, this gives $\SF^{\graph}_{\wedge_D}(\alpha_5) = 0.5$, 
    $\SF^{\graph}_{\wedge_D}(\alpha_4) = 1$, 
    $\SF^{\graph}_{\wedge_D}(\alpha_3) = 0.6$, 
    $\SF^{\graph}_{\wedge_D}(\alpha_2) = 0$ and 
    $\SF^{\graph}_{\wedge_D}(\alpha_1) = 0.3$.
    Meanwhile, by the same definitions, $\SF^{\graph'}_{\wedge_D}(\alpha_5) = 0.6$, 
    $\SF^{\graph'}_{\wedge_D}(\alpha_4) = 1$, 
    $\SF^{\graph'}_{\wedge_D}(\alpha_3) = 0.62$, 
    $\SF^{\graph'}_{\wedge_D}(\alpha_2) = 0$ and 
    $\SF^{\graph'}_{\wedge_D}(\alpha_1) = 0.31$.
    Thus, $\SF^{\graph'}_{\wedge_D}(\alpha_1) > \SF^{\graph}_{\wedge_D}(\alpha_1)$.

    \item[] Property \ref{prop:supportreinforcement}: Support Reinforcement (Counterexample).
    Let $\Args = \{ \alpha_1, \alpha_2, \alpha_3, \alpha_4 \}$ where 
    $\alpha_1 = \langle b , a \rangle$, 
    $\alpha_2 = \langle c, b \rangle$, 
    $\alpha_3 = \langle c \wedge d , \neg b \rangle$ and 
    $\alpha_4 = \langle \top, c \rangle$.
    By Definition \ref{def:supp-att}, we have that $\Atts = \{ (\alpha_3,\alpha_1) \}$ and $\Supps = \{ (\alpha_2,\alpha_1), (\alpha_4,\alpha_2), (\alpha_4,\alpha_3) \}$.
    Then let $\BS(\alpha_1) = \BS'(\alpha_1) = \BS(\alpha_3) = \BS'(\alpha_3) = \BS(\alpha_4) = 0.5$, 
    $\BS(\alpha_2) = \BS'(\alpha_2) = 1$ and $\BS'(\alpha_4) = 0.6$.
    By Definitions \ref{def:DC}, \ref{def:modular} and \ref{def:DFQuAD}, this gives $\SF^{\graph}_{\wedge_D}(\alpha_4) = 0.5$, 
    $\SF^{\graph}_{\wedge_D}(\alpha_3) = 0.6$, 
    $\SF^{\graph}_{\wedge_D}(\alpha_2) = 1$ and 
    $\SF^{\graph}_{\wedge_D}(\alpha_1) = 0.7$.
    Meanwhile, by the same definitions, $\SF^{\graph'}_{\wedge_D}(\alpha_4) = 0.6$, 
    $\SF^{\graph'}_{\wedge_D}(\alpha_3) = 0.62$, 
    $\SF^{\graph'}_{\wedge_D}(\alpha_2) = 1$ and 
    $\SF^{\graph'}_{\wedge_D}(\alpha_1) = 0.69$.
    Thus, $\SF^{\graph'}_{\wedge_D}(\alpha_1) < \SF^{\graph}_{\wedge_D}(\alpha_1)$.

    \item[] Property \ref{prop:attackmonotonicity}: Attack Monotonicity (Counterexample).
    Let $\Args = \{ \alpha_1, \alpha_2, \alpha_3, \alpha_4 \}$ and $\Args' = \Args \cup \{ \alpha_5 \}$ where 
    $\alpha_1 = \langle b , a \rangle$, 
    $\alpha_2 = \langle \top, b \rangle$, 
    $\alpha_3 = \langle c \wedge d , b \rangle$, 
    $\alpha_4 = \langle c, \neg b \rangle$ and 
    $\alpha_5 = \langle \top, c \rangle$.
    By Definition \ref{def:supp-att}, we have that $\Atts = \Atts' = \{ (\alpha_4,\alpha_1) \}$, $\Supps = \{ (\alpha_2,\alpha_1), (\alpha_3,\alpha_1) \}$ and $\Supps' = \Supps \cup \{ (\alpha_5,\alpha_3), (\alpha_5,\alpha_4) \}$.
    Then let $\BS(\alpha_1) = \BS'(\alpha_1) = \BS(\alpha_3) = \BS'(\alpha_3) = \BS'(\alpha_5) = 0.5$, 
    $\BS(\alpha_2) = \BS'(\alpha_2) = 0$ and 
    $\BS(\alpha_4) = \BS'(\alpha_4) = 1$.
    By Definitions \ref{def:DC}, \ref{def:modular} and \ref{def:DFQuAD}, this gives 
    $\SF^{\graph}_{\wedge_D}(\alpha_4) = 1$, 
    $\SF^{\graph}_{\wedge_D}(\alpha_3) = 0.5$, 
    $\SF^{\graph}_{\wedge_D}(\alpha_2) = 0$ and 
    $\SF^{\graph}_{\wedge_D}(\alpha_1) = 0.25$.
    Meanwhile, by the same definitions, $\SF^{\graph'}_{\wedge_D}(\alpha_5) = 0.5$, 
    $\SF^{\graph'}_{\wedge_D}(\alpha_4) = 1$, 
    $\SF^{\graph'}_{\wedge_D}(\alpha_3) = 0.6$, 
    $\SF^{\graph'}_{\wedge_D}(\alpha_2) = 0$ and 
    $\SF^{\graph'}_{\wedge_D}(\alpha_1) = 0.3$.
    Thus, $\SF^{\graph'}_{\wedge_D}(\alpha_1) > \SF^{\graph}_{\wedge_D}(\alpha_1)$.

    \item[] Property \ref{prop:supportmonotonicity}: Support Monotonicity (Counterexample).
    Let $\Args = \{ \alpha_1, \alpha_2, \alpha_3 \}$ and $\Args' = \Args \cup \{ \alpha_4 \}$ where 
    $\alpha_1 = \langle b , a \rangle$, 
    $\alpha_2 = \langle c, b \rangle$, 
    $\alpha_3 = \langle c \wedge d, \neg b \rangle$ and 
    $\alpha_4 = \langle \top, c \rangle$.
    By Definition \ref{def:supp-att}, we have that $\Atts = \Atts' = \{ (\alpha_3,\alpha_1) \}$, $\Supps = \{ (\alpha_2,\alpha_1) \}$ and $\Supps' = \Supps \cup \{ (\alpha_4,\alpha_2), (\alpha_4,\alpha_3) \}$.
    Then let $\BS(\alpha_1) = \BS'(\alpha_1) = \BS(\alpha_3) = \BS'(\alpha_3) = \BS'(\alpha_4) = 0.5$ and 
    $\BS(\alpha_2) = \BS'(\alpha_2) = 1$.
    By Definitions \ref{def:DC}, \ref{def:modular} and \ref{def:DFQuAD}, this gives 
    $\SF^{\graph}_{\wedge_D}(\alpha_3) = 0.5$, 
    $\SF^{\graph}_{\wedge_D}(\alpha_2) = 1$ and 
    $\SF^{\graph}_{\wedge_D}(\alpha_1) = 0.75$.
    Meanwhile, by the same definitions, $\SF^{\graph'}_{\wedge_D}(\alpha_4) = 0.5$, 
    $\SF^{\graph'}_{\wedge_D}(\alpha_3) = 0.6$, 
    $\SF^{\graph'}_{\wedge_D}(\alpha_2) = 1$ and 
    $\SF^{\graph'}_{\wedge_D}(\alpha_1) = 0.7$.
    Thus, $\SF^{\graph'}_{\wedge_D}(\alpha_1) < \SF^{\graph}_{\wedge_D}(\alpha_1)$.
    
    \squishend

    $\SF_{\wedge_Q}$:

    \squishlist

    \item[] Property \ref{prop:attackreinforcement}: Attack Reinforcement (Counterexample).
    Let $\Args = \{ \alpha_1, \alpha_2, \alpha_3, \alpha_4, \alpha_5 \}$ where 
    $\alpha_1 = \langle b , a \rangle$, 
    $\alpha_2 = \langle \top, b \rangle$, 
    $\alpha_3 = \langle  c \wedge d , b \rangle$, 
    $\alpha_4 = \langle c  , \neg b \rangle$ and 
    $\alpha_5 = \langle \top, c \rangle$.
    By Definition \ref{def:supp-att}, we have that $\Atts = \{ (\alpha_3,\alpha_1) \}$ and $\Supps = \{ (\alpha_2,\alpha_1),(\alpha_3,\alpha_1), (\alpha_5,\alpha_3), (\alpha_5,\alpha_4) \}$.
    Then let $\BS(\alpha_1) = \BS'(\alpha_1) = \BS(\alpha_3) = \BS'(\alpha_3) = \BS(\alpha_5) = 0.5$, 
    $\BS(\alpha_2) = \BS'(\alpha_2) = 0$, 
    $\BS(\alpha_4) = \BS'(\alpha_4) = 1$ and $\BS'(\alpha_5) = 0.6$.
    By Definitions \ref{def:DC}, \ref{def:modular} and \ref{def:QEM}, this gives $\SF^{\graph}_{\wedge_Q}(\alpha_5) = 0.5$, 
    $\SF^{\graph}_{\wedge_Q}(\alpha_4) = 1$, 
    $\SF^{\graph}_{\wedge_Q}(\alpha_3) = 0.54$, 
    $\SF^{\graph}_{\wedge_Q}(\alpha_2) = 0$ and 
    $\SF^{\graph}_{\wedge_Q}(\alpha_1) = 0.41$.
    Meanwhile, by the same definitions, $\SF^{\graph'}_{\wedge_Q}(\alpha_5) = 0.6$, 
    $\SF^{\graph'}_{\wedge_Q}(\alpha_4) = 1$, 
    $\SF^{\graph'}_{\wedge_Q}(\alpha_3) = 0.55$, 
    $\SF^{\graph'}_{\wedge_Q}(\alpha_2) = 0$ and 
    $\SF^{\graph'}_{\wedge_Q}(\alpha_1) = 0.42$.
    Thus, $\SF^{\graph'}_{\wedge_Q}(\alpha_1) > \SF^{\graph}_{\wedge_Q}(\alpha_1)$.

    \item[] Property \ref{prop:supportreinforcement}: Support Reinforcement (Counterexample).
    Let $\Args = \{ \alpha_1, \alpha_2, \alpha_3, \alpha_4 \}$ where 
    $\alpha_1 = \langle b , a \rangle$, 
    $\alpha_2 = \langle c, b \rangle$, 
    $\alpha_3 = \langle c \wedge d , \neg b \rangle$ and 
    $\alpha_4 = \langle \top, c \rangle$.
    By Definition \ref{def:supp-att}, we have that $\Atts = \{ (\alpha_3,\alpha_1) \}$ and $\Supps = \{ (\alpha_2,\alpha_1), (\alpha_4,\alpha_2), (\alpha_4,\alpha_3) \}$.
    Then let $\BS(\alpha_1) = \BS'(\alpha_1) = \BS(\alpha_3) = \BS'(\alpha_3) = \BS(\alpha_4) = 0.5$, 
    $\BS(\alpha_2) = \BS'(\alpha_2) = 1$ and $\BS'(\alpha_4) = 0.6$.
    By Definitions \ref{def:DC}, \ref{def:modular} and \ref{def:QEM}, this gives $\SF^{\graph}_{\wedge_Q}(\alpha_4) = 0.5$, 
    $\SF^{\graph}_{\wedge_Q}(\alpha_3) = 0.54$, 
    $\SF^{\graph}_{\wedge_Q}(\alpha_2) = 1$ and 
    $\SF^{\graph}_{\wedge_Q}(\alpha_1) = 0.59$.
    Meanwhile, by the same definitions, $\SF^{\graph'}_{\wedge_Q}(\alpha_4) = 0.6$, 
    $\SF^{\graph'}_{\wedge_Q}(\alpha_3) = 0.55$, 
    $\SF^{\graph'}_{\wedge_Q}(\alpha_2) = 1$ and 
    $\SF^{\graph'}_{\wedge_Q}(\alpha_1) = 0.58$.
    Thus, $\SF^{\graph'}_{\wedge_Q}(\alpha_1) < \SF^{\graph}_{\wedge_Q}(\alpha_1)$.

    \item[] Property \ref{prop:attackmonotonicity}: Attack Monotonicity (Counterexample).
    Let $\Args = \{ \alpha_1, \alpha_2, \alpha_3, \alpha_4 \}$ and $\Args' = \Args \cup \{ \alpha_5 \}$ where 
    $\alpha_1 = \langle b , a \rangle$, 
    $\alpha_2 = \langle \top, b \rangle$, 
    $\alpha_3 = \langle c \wedge d , b \rangle$, 
    $\alpha_4 = \langle c, \neg b \rangle$ and 
    $\alpha_5 = \langle \top, c \rangle$.
    By Definition \ref{def:supp-att}, we have that $\Atts = \Atts' = \{ (\alpha_4,\alpha_1) \}$, $\Supps = \{ (\alpha_2,\alpha_1), (\alpha_3,\alpha_1) \}$ and $\Supps' = \Supps \cup \{ (\alpha_5,\alpha_3), (\alpha_5,\alpha_4) \}$.
    Then let $\BS(\alpha_1) = \BS'(\alpha_1) = \BS(\alpha_3) = \BS'(\alpha_3) = \BS'(\alpha_5) = 0.5$, 
    $\BS(\alpha_2) = \BS'(\alpha_2) = 0$ and 
    $\BS(\alpha_4) = \BS'(\alpha_4) = 1$.
    By Definitions \ref{def:DC}, \ref{def:modular} and \ref{def:QEM}, this gives 
    $\SF^{\graph}_{\wedge_Q}(\alpha_4) = 1$, 
    $\SF^{\graph}_{\wedge_Q}(\alpha_3) = 0.5$, 
    $\SF^{\graph}_{\wedge_Q}(\alpha_2) = 0$ and 
    $\SF^{\graph}_{\wedge_Q}(\alpha_1) = 0.4$.
    Meanwhile, by the same definitions, $\SF^{\graph'}_{\wedge_Q}(\alpha_5) = 0.5$, 
    $\SF^{\graph'}_{\wedge_Q}(\alpha_4) = 1$, 
    $\SF^{\graph'}_{\wedge_Q}(\alpha_3) = 0.54$, 
    $\SF^{\graph'}_{\wedge_Q}(\alpha_2) = 0$ and 
    $\SF^{\graph'}_{\wedge_Q}(\alpha_1) = 0.41$.
    Thus, $\SF^{\graph'}_{\wedge_Q}(\alpha_1) > \SF^{\graph}_{\wedge_Q}(\alpha_1)$.

    \item[] Property \ref{prop:supportmonotonicity}: Support Monotonicity (Counterexample).
    Let $\Args = \{ \alpha_1, \alpha_2, \alpha_3 \}$ and $\Args' = \Args \cup \{ \alpha_4 \}$ where 
    $\alpha_1 = \langle b , a \rangle$, 
    $\alpha_2 = \langle c, b \rangle$, 
    $\alpha_3 = \langle c \wedge d, \neg b \rangle$ and 
    $\alpha_4 = \langle \top, c \rangle$.
    By Definition \ref{def:supp-att}, we have that $\Atts = \Atts' = \{ (\alpha_3,\alpha_1) \}$, $\Supps = \{ (\alpha_2,\alpha_1) \}$ and $\Supps' = \Supps \cup \{ (\alpha_4,\alpha_2), (\alpha_4,\alpha_3) \}$.
    Then let $\BS(\alpha_1) = \BS'(\alpha_1) = \BS(\alpha_3) = \BS'(\alpha_3) = \BS'(\alpha_4) = 0.5$ and 
    $\BS(\alpha_2) = \BS'(\alpha_2) = 1$.
    By Definitions \ref{def:DC}, \ref{def:modular} and \ref{def:QEM}, this gives 
    $\SF^{\graph}_{\wedge_Q}(\alpha_3) = 0.5$, 
    $\SF^{\graph}_{\wedge_Q}(\alpha_2) = 1$ and 
    $\SF^{\graph}_{\wedge_Q}(\alpha_1) = 0.6$.
    Meanwhile, by the same definitions, $\SF^{\graph'}_{\wedge_Q}(\alpha_4) = 0.5$, 
    $\SF^{\graph'}_{\wedge_Q}(\alpha_3) = 0.54$, 
    $\SF^{\graph'}_{\wedge_Q}(\alpha_2) = 1$ and 
    $\SF^{\graph'}_{\wedge_Q}(\alpha_1) = 0.59$.
    Thus, $\SF^{\graph'}_{\wedge_Q}(\alpha_1) < \SF^{\graph}_{\wedge_Q}(\alpha_1)$.

    \squishend

\end{proof}

\end{document}